\newif\ifcolt

\ifcolt %

  \makeatletter
  \def\input@path{{COLT-style/}}
  \makeatother

  \documentclass[final,12pt]{colt2026} %

  \makeatletter
  \def\input@path{{./}}
  \makeatother

  \usepackage{times}

  \makeatletter
\let\origremark\remark
\let\endorigremark\endremark

\RenewDocumentEnvironment{remark}{o}
  {%
    \IfNoValueTF{#1}{\origremark}{\origremark[#1]}%
    \normalfont
  }
  {%
    \unskip\hfill$\clubsuit$\endorigremark
  }

\makeatother

\newtheorem{assumption}[theorem]{Assumption}
\newtheorem{informalprinciple}[theorem]{Informal Principle}

  \title[Provable Learning of RHMs]{Provable Learning of Random Hierarchy Models and Hierarchical Shallow-to-Deep Chaining}

  \coltauthor{%
    \Name{Yunwei Ren} \Email{yunwei.ren@princeton.edu} \\
    \addr Princeton University
    \AND
    \Name{Yatin Dandi} \Email{yatin.dandi@epfl.ch} \\
    \addr \'Ecole Polytechnique F\'ed\'erale de Lausanne 
    \AND
    \Name{Florent Krzakala} \Email{florent.krzakala@epfl.ch} \\
    \addr \'Ecole Polytechnique F\'ed\'erale de Lausanne
    \AND
    \Name{Jason D. Lee} \Email{jasondlee@berkeley.edu} \\
    \addr University of California, Berkeley
  }

\else %

  \ifdefined\usebigfont  
  
  \documentclass[12pt]{article}
  \usepackage[margin=1.25in]{geometry}

  \else

  \documentclass[11pt]{article}
  \usepackage[margin=1in]{geometry}
  
  \fi

  \usepackage[english]{babel}
\usepackage{natbib}
\usepackage{times}

\usepackage{newpxtext,newpxmath}

\usepackage{amsthm}

\usepackage{
  mathtools,
  thmtools,
  url
}

\usepackage[
    colorlinks=true,
    linkcolor=blue,
    anchorcolor=blue,
    citecolor=blue
]{hyperref}

\newtheorem{theorem}{Theorem}[section]
\newtheorem{informalprinciple}[theorem]{Informal Principle}

\newtheorem{lemma}{Lemma}[section]

\newtheorem{corollary}[lemma]{Corollary}
\newtheorem{definition}[lemma]{Definition}
\newtheorem{assumption}[lemma]{Assumption}

\newenvironment{remark}[1][Remark]
  {
  \begin{proof}[\textnormal{\textbf{#1}}]}
  {\end{proof}}

  \usepackage{makecell}

  \title{Provable Learning of Random~Hierarchy~Models and Hierarchical Shallow-to-Deep Chaining}
  \author{
  \begin{tabular}{cc}
    \makecell{ Yunwei Ren \\ \normalsize{Princeton University} \\ \normalsize{\texttt{yunwei.ren@princeton.edu}} } &
    \makecell{ Yatin Dandi \\ \normalsize{\'Ecole Polytechnique F\'ed\'erale de Lausanne} \\ \normalsize{\texttt{yatin.dandi@epfl.ch}} } \\ \\
    \makecell{ Florent Krzakala \\ \normalsize{\'Ecole Polytechnique F\'ed\'erale de Lausanne} \\ \normalsize{\texttt{florent.krzakala@epfl.ch}} } &
    \makecell{ Jason D.~Lee \\ \normalsize{University of California, Berkeley} \\ \normalsize{\texttt{jasondlee@berkeley.edu}}  } \\
  \end{tabular}
  }
\fi

\usepackage{
  bbm, 
  bm
}

\usepackage{tikz}
\usetikzlibrary{positioning,fit,shapes.geometric,calc}
\usepackage{caption}
\usepackage[dvipsnames]{xcolor}
\usepackage[normalem]{ulem}
\usepackage[shortlabels]{enumitem}

\usepackage{thm-restate}

\newcommand\locallabel[1]{\label{\currentprefix:#1}}
\newcommand\localref[1]{\ref{\currentprefix:#1}}

\newcommand{\mbb}{\mathbb}
\newcommand{\mrm}{\mathrm}
\newcommand{\mbf}{\bm}
\newcommand{\mcal}{\mathcal}
\newcommand{\tnbf}[1]{\textnormal{\textbf{#1}}}

\newcommand{\eps}{\varepsilon}
\newcommand{\A}{\mbf{A}}

\newcommand{\e}{\mbf{e}}
\newcommand{\G}{\mbf{G}}
\newcommand{\h}{\mbf{h}}
\newcommand{\p}{{\mbf{p}}}
\newcommand{\q}{{\mbf{q}}}
\newcommand{\Q}{{\mbf{Q}}}
\newcommand{\R}{\mathbb{R}}

\newcommand{\Id}{\mbf{I}}
\newcommand{\w}{{\mbf{w}}}
\newcommand{\W}{{\mbf{W}}}
\newcommand{\x}{{\mbf{x}}}
\newcommand{\X}{{\mbf{X}}}
\newcommand{\y}{\mbf{y}}
\newcommand{\z}{\mbf{z}}

\newcommand{\bE}{\mbf{E}}
\newcommand{\bP}{\mbf{P}}
\newcommand{\cD}{\mcal{D}}
\newcommand{\cF}{\mcal{F}}

\newcommand{\cP}{\mcal{P}}
\newcommand{\cQ}{\mcal{Q}}
\newcommand{\cR}{\mcal{R}}
\newcommand{\cS}{\mcal{S}}

\newcommand{\cV}{\mcal{V}}
\newcommand{\cX}{\mcal{X}}
\newcommand{\cZ}{\mcal{Z}}
\newcommand{\bbN}{\mbb{N}}
\newcommand{\rmR}{\mrm{R}}

\newcommand{\bDelta}{\mbf{\Delta}}
\newcommand{\bmu}{\mbf{\mu}}
\newcommand{\bPhi}{\mbf{\Phi}}
\newcommand{\bomega}{\mbf{\omega}}
\newcommand{\btau}{\mbf{\tau}}

\newcommand{\inprod}[2]{\left\langle #1, #2\right\rangle}
\newcommand{\norm}[1]{\left\|#1\right\|}
\newcommand{\abs}[1]{ {\left| #1 \right|} }
\newcommand{\braces}[1]{ \left\{ #1 \right\} }
\newcommand{\inv}{^{-1}}
\newcommand{\trans}{^\top}
\newcommand{\ps}[1]{^{(#1)}}
\newcommand{\indi}{\mathbbm{1}}
\DeclareMathOperator{\supp}{supp}
\newcommand{\E}{\mathop{\mathbb{E\/}}}
\renewcommand{\P}{\mathop{\mathbb{P\/}}}
\newcommand{\Var}{\mathop{\mathbf{Var\/}}}
\newcommand{\Cov}{\mathop{\mathbf{Cov\/}}}

\DeclareMathOperator*{\argmax}{argmax}
\DeclareMathOperator{\poly}{poly}

\newcommand{\One}{\mbf{1}}
\newcommand{\Gaussian}[2]{\mcal{N}\left(#1, #2\right)}
\newcommand{\Unif}{\mrm{Unif}}
\newcommand{\Proj}{\mbf{\Pi}}

\newcommand{\Loss}{\mcal{L}}
\newcommand{\Tmp}{\texttt{Tmp}}
\newcommand{\RF}{\mrm{RF}}
\renewcommand{\wr}{\mrm{wr}}
\newcommand{\wor}{\mrm{wor}}

\ifcolt %
  
\else
  
\fi

\ifdefined\usebigfont

\usepackage{times}
\usepackage[fontsize=13pt]{scrextend}
\makeatletter
\@ifpackageloaded{geometry}{\AtBeginDocument{\newgeometry{letterpaper,left=1.56in,right=1.56in,top=1.71in,bottom=1.77in}}}{\usepackage[letterpaper,left=1.56in,right=1.56in,top=1.71in,bottom=1.77in]{geometry}}
\AtBeginDocument{\newgeometry{letterpaper,left=1.56in,right=1.56in,top=1.71in,bottom=1.77in}}
\usepackage{hyperref} %
\else
\fi
\begin{document}

\maketitle

\begin{abstract}
  The empirical success of deep learning is often attributed to deep networks' ability to exploit hierarchical structure in data, constructing increasingly complex features across layers. 
  Yet despite substantial progress in deep learning theory, most \emph{optimization} results still focus on networks with only two or three layers, leaving the theoretical understanding of hierarchical learning in genuinely deep models limited. This leads to a natural question: 
  can we prove that deep networks, trained with gradient-based methods and standard input-label pairs, can efficiently exploit hierarchical structure?

  In this work, we consider Random Hierarchy Models --- a hierarchical context-free grammar introduced by \cite{cagnetta_how_2024} and conjectured to separate deep and shallow networks. 
  We prove that, under mild conditions, a deep convolutional network can be efficiently trained to learn this function class. 
  Our proof builds on a general observation: if intermediate layers can receive clean signal from the labels and the relevant features are weakly identifiable, then layerwise training each individual layer suffices to hierarchically learn the target function.
\end{abstract}

\ifcolt
\fi

\section{Introduction}

A common heuristic explanation for the effectiveness of deep learning is that deeper models can more efficiently
exploit the hidden hierarchical structure of the target by building more complex features in later layers
from the simpler features learned in earlier layers (\cite{lecun_deep_2015,goodfellow_deep_2016}). 
Formally establishing this heuristic is one of the central goals of deep learning theory.

However, despite recent progress in deep learning theory, most theoretical results still concern networks with two or three 
layers and/or unstructured inputs (\cite{ben_arous_online_2021,bietti_learning_2022,damian_neural_2022,%
nichani_provable_2023,abbe_sgd_2023,dandi_how_2024,dandi_computational_2025,ren_emergence_2025}), 
limiting our understanding on how deep models can exploit hierarchical structure of data.
For deeper networks, existing results are largely confined to 
approximation/representation power (\cite{telgarsky_benefits_2016,yarotsky_phase_2020,lu_deep_2021}), 
the kernel/lazy training regime (\cite{allen-zhu_learning_2019}), 
linear networks (\cite{zou_global_2019}), 
and/or infinite-width mean-field limits (\cite{lu_mean_2020,nguyen_rigorous_2023}). 
For models with more than three layers and potentially hierarchically structured inputs, 
positive \emph{optimization} guarantees outside the kernel and mean-field regimes are rare, 
in part due to the complex training dynamics induced by the interactions across layers, and often require
heavily modified algorithms, involved analysis, and/or access to intermediate information (\cite{allen-zhu_backward_2023,wang_learning_2025,daniely_deep_2026}).
This raises a natural question: Can we provably exploit the deep hierarchical structure of data using 
a deep network?

Motivated by the above question, we study Random Hierarchy Models (RHMs), a class of probabilistic context-free grammars 
(PCFGs) with hierarchical structures. Introduced by \cite{cagnetta_how_2024}, RHMs provide an example where deeper models 
appear to require far fewer samples than shallow ones. 
More specifically, \cite{cagnetta_how_2024} conjectured, based on numerical results and heuristic arguments, that 
to learn an $L$-level RHM with branching factor $s$ and $m$ production rules per token, the sample complexity
of an $L$-level model scales with $m^L$, while that of a shallow model scales with $m^{s^L}$. 
See Section~\ref{subsec: rhm} and Figure~\ref{fig: rhm tree} for the formal definition of RHMs. 
Here, $m, s, L \in \bbN$ are RHM parameters. 
In particular, note that when $m, s = O(1), L \sim \log d$ or $m,s \sim \log d, L \sim \log d / \log\log d$,
where $d := s^L$ is the input length,
the sample complexity of a deep model ($\sim m^L$) is polynomial in $d$, whereas that of a 
shallow model ($\sim m^{s^L}$) is exponential. In short, RHMs are conjectured to separate deep and shallow networks. 
In this work, we formally prove the positive part of this conjecture. 
\begin{theorem}[Informal version of Theorem~\ref{thm: opt main}]
  Consider a non-degenerate $L$-level RHM with $m$ production rules per token. A polynomially wide $L$-layer convolutional network, trained layerwise  with gradient descent, 
  can efficiently learn this RHM using $O(m^{(1+o(1))L})$ samples and polynomially many gradient steps.
\end{theorem}

Our proof is based on Informal Principle~\ref{principle: shallow to deep chaining}, which gives criteria under 
which hierarchical learning can be achieved via layerwise training a deep network.\footnote{To illustrate the potential generality of this principle, 
we also include in Appendix~\ref{sec: deep quad} a different task to which it applies.} 
A motivation for pursuing such a strategy comes from the ubiquitous successful use of 
``global-to-local'' reductions in areas such as complexity theory (\cite{arora_computational_2016}), coding theory (\cite{yekhanin_locally_2012}), high-dimensional probability (\cite{van_handel_probability_2016}), 
and statistical physics under the frameworks of coarse-graining and renormalization group \cite{wilson1971renormalization}.
Our proof is inspired by the above examples: 
We will first prove optimization results for learning a single layer,
and then chain them to obtain the optimization guarantee for our deep learner model.

Heuristically, we call a function or data-generating process hierarchical if it can be represented as a composition of a sequence of simpler functions or processes. A learning process is said to be 
hierarchical if it efficiently exploits this compositional structure.
The learning task can become easier upon the recovery of intermediate functions due to, for example, their 
reduced dimension and/or increased low-degree correlations with the target, as in tree-structured distributions considered in \cite{cagnetta_how_2024}, and the progressive reduction in ``effective dimension"  in the class of targets considered in \cite{dandi_computational_2025}.

\begin{informalprinciple}[Shallow-to-deep chaining]
  \label{principle: shallow to deep chaining}
  Suppose that the data is hierarchical.
  If (i) the function output is ``correlated'' with its lower-level parts, 
  (ii) the signal received by the lower-level parts from the output is ``clean'', 
  in the sense that the lower-level parts of the learner will not overfit the higher-level part of the target
  when properly trained to fit the signal,
  (iii) the lower-level features are ``(weakly) identifiable'', 
  then shallow-to-deep chaining is possible, i.e., we can efficiently learn this target by layerwise training a 
  suitable deep network. 
\end{informalprinciple}
It is intended for the meaning of these three conditions to be high-level and problem-dependent. 
Nevertheless, we will briefly explain why intuitively these conditions can be useful. 

First, condition~(i) ensures that the lower portion of the learner can learn something about the target before the 
higher portions have learned anything. Without such a condition, learning a hierarchical
function can potentially be difficult even for deep networks (\cite{li_noise_2025}). 
A related notion also appeared in \cite{dandi_computational_2025} under the name ``compositional Information exponent".
See also \cite{allen-zhu_what_2019,allen-zhu_backward_2023,malach_is_2019}. 

Condition~(ii) ensures there is no overfitting bias, which is unfixable with layerwise training and can 
potentially make the error blow up across layers. 
(See \cite{allen-zhu_backward_2023} for an example where layerwise training fails.)
Note that this condition depends not only on the target function, but also the learner model, training algorithm, and/or 
regularization, since one can proactively restrict the power of each single layer to prevent overfitting and make this 
condition hold. 
There are various scenarios in which this condition can hold: In the RHM case, we rely on certain 
conditional independence property, while in the example in Appendix~\ref{sec: deep quad}, orthogonality is used. 
Moreover, this condition holds almost trivially when one has access to the intermediate results, as in 
curriculum learning and Chain-of-Thought learning (\cite{wang_learning_2025,kim_transformers_2024,daniely_deep_2026}).

Condition~(iii) means that when the model fits the lower portion of the target in value, 
it also learns useful lower-level features that can be reused by later layers to learn higher portion of the target. 
In many cases, including the RHMs considered in this work, it is not necessary and/or possible to recover
a specific set of lower-level features, and it often suffices to learn some set of lower-level features with certain 
problem-dependent properties, such as recovering the ground-truth low-level features through simple transformations.

Finally, we remark that we do \emph{not} claim Informal Principle~\ref{principle: shallow to deep chaining} is a 
necessary condition for successful hierarchical learning. Instead, we only claim it to be a reasonably general sufficient 
condition for successful hierarchical learning via layerwise training that covers non-trivial examples. 

\paragraph{Organization.}
The rest of this paper is organized as follows. We discuss the related work in the remaining part of the introduction.
In Section~\ref{sec: setup}, we formally define RHMs, and describe our learner model and training algorithm. 
Then, as a warm-up, we present in Section~\ref{sec: warm up approx} three approximation results. In Section~\ref{sec: opt results and proof sketch}, 
we formally state our main optimization result and give a proof sketch. Finally, we conclude in Section~\ref{sec: conclusion}.
The proofs can be found in Appendix~\ref{sec: opt proofs} and \ref{sec: signal lower bounds}. Appendix~\ref{sec: deep 
quad} contains another application of Informal Principle~\ref{principle: shallow to deep chaining}.

\subsection{Related work}

\paragraph{Depth separation.}
Beginning with \cite{eldan_power_2016,telgarsky_benefits_2016}, there is a long line of 
work that aims to separate the power of deep and shallow networks (\cite{daniely_depth_2017,safran_depth-width_2017,%
safran_depth_2019,malach_connection_2021,venturi_depth_2022,safran_optimization-based_2022,ren_depth_2023,safran_depth_2024}). 
Most of them concern the separation in approximation powers for two- versus three-layer networks.
\cite{telgarsky_benefits_2016} proves a separation between $\Theta(k^3)$- and $\Theta(k)$-depth networks,
but again the result is purely about approximation and relies on extreme oscillations. 
\cite{safran_optimization-based_2022,ren_depth_2023} prove optimization-based separations, but only for two-layer versus three-layer networks.

\paragraph{Hierarchical learning.}
Several recent works study learning hierarchical functions using networks with more than 
two layers. \cite{allen-zhu_what_2019,wang_learning_2023,nichani_provable_2023,dandi_computational_2025,fu_learning_2025} 
consider target functions that are compositions of two simpler functions and show that three-layer networks can 
efficiently learn such functions, outperforming two-layer networks and kernel methods by exploiting the compositional 
structure. 
Results addressing general depth, by contrast, are relatively scarce.
The work of \cite{mossel2016deep} introduced a class of ``generative hierarchical models'' similar to our work, with the data generated through a tree started from a label at the root. Subsequently they proved a separation in efficient recovery of the label between a class of deep and shallow algorithms. However, the relationship between this algorithmic framework and gradient-based training of deep neural networks was left unresolved.
More recently, for any constant $L>1$, \cite{dandi_computational_2025} proposed a class of target functions conjectured to require depth strictly greater than $L$. Their analysis, however, yields only a partial, conditional weak-recovery guarantee for this class.
\cite{daniely_deep_2026} proves deep networks can perform hierarchical learning when ground-truth intermediate results are provided to the training algorithm.

In a different vein, \cite{allen-zhu_backward_2023} studied learning deep quadratic functions using 
$\omega(1)$-depth residual networks. Our target class and techniques are substantially different from theirs. Conceptually,
their setting is similar to Informal Principle~\ref{principle: shallow to deep chaining} with only conditions~(i) and (iii). 
While their results show the power of $\omega(1)$-depth networks, the absence of condition~(ii) necessitates reliance on the backward 
feature correction mechanism and a highly specialized algorithm, and precludes a simple shallow-to-deep chaining argument.

\paragraph{Context-free grammars and Random Hierarchy Models.}
Since the mid-20th century, context-free grammars (CFGs, \cite{chomsky_certain_1959}) have been widely used in 
various areas of computer science, including complexity theory (\cite{sipser_introduction_2013}) and language 
modeling (\cite{jurafsky_speech_2009}). 
Recently, CFGs have also been used to study the capabilities of modern language models (\cite{zhao_transformers_2023,%
allen-zhu_physics_2025}). In \cite{cagnetta_how_2024}, the authors introduced Random Hierarchy Models (RHMs), a 
special class of CFGs with hierarchical structure, to empirically study how deep networks learn compositional
structure. 
This model has since been used to study various other phenomena in deep learning (\cite{cagnetta_towards_2024,cagnetta_scaling_2025,cagnetta_learning_2025,cagnetta_deriving_2026,sclocchi_probing_2025}). 
\cite{mei_u-nets_2024} analyzes a similar model and shows that deep networks can simulate the belief-propagation algorithm. 
None of these works, except \cite{mei_u-nets_2024} which proves an approximation result, provide formal approximation or learning guarantees.

\section{Setup}
\label{sec: setup}

In this section, we first define Random Hierarchy Models, describe the learning task and our learning model. 
After that, we formally state our main results. 

\paragraph{Notation.} 
For any vector $\x$, $\norm{\x}$ denotes the Euclidean norm of $\x$. For any matrix $\A$, $\norm{\A}_2$ denotes the 
spectral norm of $\A$ and $\norm{\A}_F$ denotes the Frobenius norm. 
For $a, b, \delta \in \R$, we write $a = b \pm \delta$ for $|a - b| \le \delta$. 
If $\mbf{a}, \mbf{b}$ are vectors or matrices, then $\mbf{a} = \mbf{b} \pm_2 \delta$ means $\norm{\mbf{a} - \mbf{b}}_2 \le \delta$. 
For any $n \in \bbN$, we define $[n] := \{1, 2, \dots, n\}$ and $[n]_0 := \{0, 1, \dots, n\}$.

\subsection{Random Hierarchy Models}
\label{subsec: rhm}

Random Hierarchy Models (RHMs) are special cases of probabilistic context-free grammars (PCFGs). 
A context-free grammar (CFG), in short, is a set $\cV$ of symbols, together with a collection $\cR$ of production rules. 
A production rule is a pair $(\nu, \bmu) \in \cV \times \cV^*$. Semantically, it means that we can replace the symbol 
$\nu$ with the sequence of symbols $\bmu$. Usually, a special symbol from $\cV$ is chosen to represent the starting symbol. 
To generate a sentence using a CFG, we start with the sentence consisting of a single starting symbol, 
and then repeatedly rewrite the sentence using the production rules until no rewriting can be done. 
A CFG is said to be probabilistic if the production rules used in each step are randomly sampled 
according to some distribution.

RHMs are PCFGs where we have multiple levels, each level has its own 
vocabulary, and the production rules always rewrite a level-$l$ symbol to $s$ level-$(l+1)$ symbols, where $s$
is a fixed positive integer. Equivalently, we can view RHMs as a variant of graphical models, where the underlying 
graph is a complete $s$-ary tree and the transition probability is associated with the level instead of the edges. 
See~Figure~\ref{fig: rhm tree} for an example of the tree description of a RHM with $3$ levels.
At a high level, a RHM can also be viewed as a Markov process with certain local structure over a space 
of growing dimension.

Note that, under certain non-ambiguity conditions, learning a RHM is easy if we have access to the latent intermediate tokens, 
since we can then directly recover the (intermediate) production rules. Unfortunately, without this intermediate information, exact recovery of the production rules is impossible in general. Nevertheless, we will show that using only input--label data, each layer of a deep convolutional 
network can still learn proxies for the true production rules of the corresponding RHM level, and that these proxies are sufficient for 
the model to hierarchically learn the RHM.

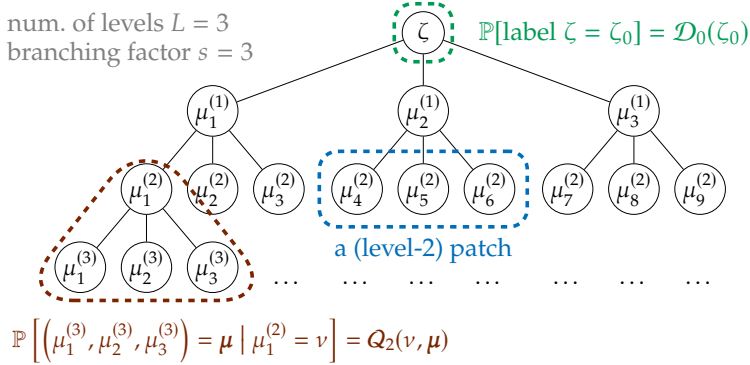
\begin{figure}[htbp]
  \centering  

  \begin{minipage}{0.62\textwidth}
    \scalebox{0.80}{
    \begin{tikzpicture}[
      level distance=13mm,
      level 1/.style={sibling distance=35mm},
      level 2/.style={sibling distance=11mm},
      level 3/.style={sibling distance=11mm},
      every node/.style={circle, draw, minimum size=7mm, inner sep=0.1mm},
      edge from parent/.style={draw},
    ]

    \def\ellipsisYShift{-3mm}

    \node (root) {$\zeta$}
      child { node{$\mu\ps{1}_1$}
        child { node (mu21) {$\mu\ps{2}_1$} 
          child { node (mu31) {$\mu\ps{3}_1$} }
          child { node (mu32) {$\mu\ps{3}_2$} }
          child { node (mu33) {$\mu\ps{3}_3$} }
        }
        child { node{$\mu\ps{2}_2$} }
        child { node{$\mu\ps{2}_3$} 
          child[edge from parent/.style={draw=none}] {node[yshift=\ellipsisYShift, xshift=1.5mm, draw=none]{$\cdots$}}
        }
      }
      child { node{$\mu\ps{1}_2$}
        child { node (mu24) {$\mu\ps{2}_4$} 
          child[edge from parent/.style={draw=none}] {node[yshift=\ellipsisYShift, draw=none]{$\cdots$}}
        }
        child { node (mu25) {$\mu\ps{2}_5$} 
          child[edge from parent/.style={draw=none}] {node[yshift=\ellipsisYShift, draw=none]{$\cdots$}}
        }
        child { node (mu26) {$\mu\ps{2}_6$} 
          child[edge from parent/.style={draw=none}] {node[yshift=\ellipsisYShift, draw=none]{$\cdots$}}
        }
      }
      child { node{$\mu\ps{1}_3$}
        child { node{$\mu\ps{2}_7$} 
          child[edge from parent/.style={draw=none}] {node[yshift=\ellipsisYShift, draw=none]{$\cdots$}}
        }
        child { 
          node{$\mu\ps{2}_8$} 
          child[edge from parent/.style={draw=none}] {node[yshift=\ellipsisYShift, draw=none]{$\cdots$}}
        }
        child { 
          node{$\mu\ps{2}_9$} 
          child[edge from parent/.style={draw=none}] {node[yshift=\ellipsisYShift, draw=none]{$\cdots$}}
        }
      };

      \node[anchor=north west, font=\large, align=left, shape=rectangle, draw=none]
        at ([xshift=-8mm,yshift=0mm]current bounding box.north west)
        {\color{gray}num.~of levels $L = 3$ \\ \color{gray}branching factor $s = 3$};
      
      \node (label) [
        fit=(root), inner sep=1.25mm, shape=rectangle,
        line width=0.6mm, dashed,
        line join=round, rounded corners=3mm, color=ForestGreen
      ] {};
      \node [
        right=of label, xshift=-6mm, yshift=0mm,
        shape=rectangle, color=ForestGreen, draw=none
      ] {
        \large $\P[ \text{label } \zeta = \zeta_0 ] = \cD_0(\zeta_0)$
      };

      \def\padx{8mm} 
      \def\pady{2.5mm} 
      \draw[line width=0.6mm, dashed, rounded corners=10mm, line join=round, color=Brown]
        ([xshift=-\padx, yshift=-\pady] mu31.south west) --
        ([xshift= \padx, yshift=-\pady] mu33.south east) --
        ([yshift=5mm] mu21.north) -- cycle;
      \node (left_subtree_box) [fit=(mu21)(mu31)(mu32)(mu33), inner sep=0, shape=rectangle, draw=none] {}; %
      \node [
        below=of left_subtree_box, xshift=14mm, yshift=6mm,
        shape=rectangle, color=Brown, draw=none
      ] {
        $\P\left[ \left(\mu\ps{3}_1, \mu\ps{3}_2, \mu\ps{3}_3\right) = \bmu \;\big|\; \mu\ps{2}_1 = \nu \right]
        = \cQ_2(\nu, \bmu)$
      };

      \node (mid_patch) [
        fit=(mu24)(mu25)(mu26), inner sep=2mm, shape=rectangle,
        line width=0.6mm, dashed,
        line join=round, rounded corners=3mm, color=NavyBlue
      ] {}; 
      \node [
        below=of mid_patch, xshift=0mm, yshift=10mm,
        shape=rectangle, color=NavyBlue, draw=none
      ] {
        \large a (level-$2$) patch
      };
    \end{tikzpicture}
    }
  \end{minipage}%
  \hfill
  \begin{minipage}{0.35\textwidth}
    \captionsetup{format=plain}
    \caption{
      Tree description of a $3$-level RHM instance with branching factor $3$. 
      Each node is a random variable. The root node is the label and the leaves are the input sequence. 
      At level $l$, the conditional probability of the children patch being $\bmu \in \cV_{l+1}^s$ 
      conditioned on the parent token being $\nu \in \cV_l$ is given by $\cQ_l(\nu, \bmu)$.
    }
    \label{fig: rhm tree}
  \end{minipage}
\end{figure}

Formally, we define a RHM as follows. 
This definition is slightly more general than the CFG description in \cite{cagnetta_how_2024}.
We will soon define RHM induced by a CFG as a RHM whose transition probabilities are induced by production rules. 
\begin{definition}[Random Hierarchy Model]
  \label{def: rhm}
  Let $L, s \in \bbN$ be parameters. 
  An $L$-level \tnbf{random hierarchy model (RHM)} instance with branching factor $s$ is a tuple
  $\left( (\cV_l)_{l \in [L]_0}, \cD_0, ( \cQ_l )_{l \in [L-1]_0} \right)$, where 
  \begin{itemize}
    \item $\cV_l$ is a set representing the vocabulary at level $l$,
    \item $\cD_0$ is a probability distribution over $\cV_0$ representing the distribution of the labels,  
    \item $\cQ_l: \cV_l \times \cV_{l+1}^s \to [0, 1]$ encodes the transition probabilities at level $l$, 
      i.e., for each $\nu \in \cV_l$, $\cQ_l(\nu, \cdot)$ is a probability distribution over  
      the collection of possible patches $\cV_{l+1}^s$. 
  \end{itemize}

  To generate a sentence using a RHM instance, we first sample the label $\zeta \in \cV_0$ according to $\cD_0$ to 
  form the level-$0$ sentence $\bmu\ps{0} = (\zeta) \in \cV_0^1$. Then, for each $l \in [L-1]_0$, we rewrite 
  the level-$l$ sentence $\bmu\ps{l} = ( \mu\ps{l}_1, \dots, \mu\ps{l}_{s^l} )$ into 
  a level-$(l+1)$ sentence $\bmu\ps{l+1} \in \cV_{l+1}^{s^{l+1}}$ by replacing each 
  $\mu\ps{l}_k$ with an independent sample from $\cQ_l(\mu\ps{l}_k, \cdot)$. The final sentence $\bmu\ps{L}$
  is the generated sentence. 
\end{definition}

In this work, we are mainly interested in RHMs induced by a collection of production rules, which is also how RHMs 
are originally formulated in \cite{cagnetta_how_2024}.
\begin{definition}[CFG-induced RHM]
  \label{def: cfg rhm}
  Let $L, s \in \bbN$ be parameters. For each $l \in [L-1]_0$, let $\cR_l \subset \cV_l \times \cV_{l+1}^s$
  represents a collection of production rules from level-$l$ symbols to level-$(l+1)$ patches. 
  We say the RHM $ ( (\cV_l)_{l\in[L]_0}, \cD_0, (\cQ_l)_{l \in [L-1]_0} )$ is 
  \tnbf{induced by the production rules $(\cR_l)_{l \in [L-1]_0}$}
  if $\cD_0 = \Unif(\cV_0)$ and 
  \[
    \cQ_l( \nu, \bmu ) 
    = \frac{\indi\braces{ (\nu, \bmu) \in \cR_l } }{ 
      \abs{ \braces{ (\nu, \bmu') \in \cR_l \mid \bmu' \in \cV_{l+1}^s } }
    } , \quad 
    \forall \nu \in \cV_l,\, \bmu \in \cV_{l+1}^s, 
  \]
  and the denominator is always nonzero. 
  Namely, $\cQ_l(\nu, \cdot)$ is the uniform distribution of its support, 
  which is determined by the given production rules $\cR_l$.
\end{definition}

\begin{remark}
  In the original formulation of RHMs in \cite{cagnetta_how_2024}, they explicitly assume 
  the production rules $\cR_l$ are randomly sampled. We do not require this and one should view $\cR_l$ and $\cQ_l$
  as deterministic objects that potentially satisfy certain mild pseudo-randomness assumption (cf.~Assumption~\ref{assumption: rhm}).

  When a RHM is induced by production rules, we can impose additional constraints on the RHM by imposing constraints
  on the production rules. In this paper, we will consider CFG-induced RHMs satisfying the following two conditions.
  
  Let $V, m \in \bbN$. 
  We say the CFG-induced RHM $( (\cV_l)_{l\in[L]_0}, \cD_0, (\cR_l)_{l \in [L-1]_0})$ is \tnbf{$(V, m)$-uniform} if 
  $|\cV_0| = |\cV_1| = \cdots = |\cV_L| = V$ and for every $l \in [L-1]_0$, every $\nu \in \cV_l$ has exactly 
  $m$ (distinct) production rules, i.e., $| \braces{ (\nu, \bmu') \in \cR_l \mid \bmu' \in \cV_{l+1}^s } | = m$. 
  Note that in this case, we have 
  \[
    \cQ_l( \nu, \bmu ) 
    =  \indi\braces{ (\nu, \bmu) \in \cR_l } / m  , \quad 
    \forall \nu \in \cV_l,\, \bmu \in \cV_{l+1}^s.
  \]
  For notational simplicity, we will write $\cP_{l+1} := \bigcup_{\nu \in \cV_l} \supp \cQ_l( \nu, \cdot)$. 
  That is, $\cP_{l+1}$ represents the collection of possible patches at level $l+1$. 

  We say the RHM is \tnbf{non-ambiguous} if for every level-$(l+1)$ patch $\bmu \in \cV_{l+1}^s$, there is at most 
  one level-$l$ symbol $\nu \in \cV_l$ that can generate it. That is, for every $l \in [L-1]_0$ and 
  $\bmu \in \cV_{l+1}^s$, $| \{ (\nu', \bmu') \in \cR_l \mid \bmu' =  \bmu \}| \le 1$. We say two patches 
  $\bmu, \bmu' \in \cV_{l+1}^s$ are \tnbf{synonyms} if they can be generated by the same level-$l$ symbol.
  Non-ambiguity is also assumed in \cite{cagnetta_how_2024}.
\end{remark}

\subsection{Learner model}

We assume the branching factor $s$ is known, so that we can design the model architecture to match
the tree structure of the RHM by grouping $s$ tokens into a patch. Note that the production rules are still unknown, 
and since the intermediate nodes are unobservable, learning the production rules is a nontrivial task.
As the RHM has an inherent $L$-layer structure and is shift-invariant at the patch level, it is natural to choose 
our learner model to be an $L$-layer convolutional network. 

\begin{figure}[tbp]
  \centering
  \begin{tikzpicture}[
    node distance=0mm and 25mm
  ]
    \node (token_embed) {  
      $\left. \begin{matrix} \h\ps{l}_{(k-1)s+1} \in \R^{d_e} \\ \vdots \\ \h\ps{l}_{ks} \in \R^{d_e} \end{matrix} \right\}$ 
    };
    \node [below=of token_embed, align=center] {embeddings of the tokens \\ in the $k$th patch of level $l$};

    \node (patch_embed) [right=of token_embed, xshift=-7mm] {$\x_{\bmu_k\ps{l}} \in \R^{d_x}$ };
    \node [below=of patch_embed, align=center, yshift=-6mm] {patch embedding};

    \node (output) [right=of patch_embed, xshift=15mm] {$\h\ps{l-1}_k \in \R^{d_e}$ };
    \node [below=of output, align=center, yshift=-6mm] {
        level-$l$ output;
        \\level-$(l-1)$ token embedding.
      };

    \draw [|->] (token_embed) to 
      node [above, align=center] {$\bPhi\ps{l}(\cdot)$}
      (patch_embed);

    \draw [|->] (patch_embed) to 
      node [above, align=center] {$\W\ps{l}\x_{\bmu_k\ps{l}} $} 
      node [below, align=center] { potentially normalized } 
      (output);
  \end{tikzpicture}
  \caption{
    One layer of our learner model. For each level $l$, the model acts on each patch $(\h\ps{l}_{(k-1)s+1}, \dots, 
    \h\ps{l}_{ks})$ independently using the same set of weights $\W\ps{l} \in \R^{d_y \times d_x}$ and nonlinearity
    $\bPhi\ps{l}: \R^{s d_e} \to \R^{d_e}$. The output on the $k$-th level-$l$ patch will be used as the embedding of 
    $k$-th token at level $l-1$.
  }
  \label{fig: level-l of the learner model}
\end{figure}
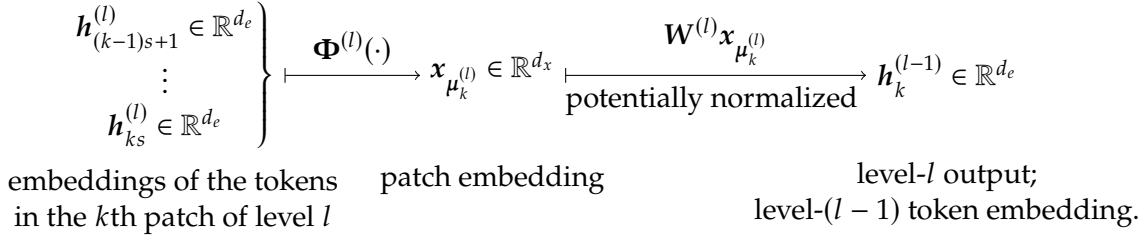

See Figure~\ref{fig: level-l of the learner model} for a schematic description of one layer of our learner model. 
We formally define our model below. 
Fix $l \in [L]$. Let $ \h_1\ps{l}, \dots, \h_{s^l}\ps{l} \in \R^{d_e}$ denote the input sequence at level $l$, 
where $d_e \ge V$ is the token embedding size. When $l = L$, $\h_i\ps{L}$ is the one-hot encoding of the $i$th input token. 
When $l < L$, these are the outputs of the previous layer. 
Because of the $s$-ary tree structure of the RHM, it is natural to group the tokens into patches of length $s$. 
We then apply a nonlinear transform $\bPhi\ps{l}: \R^{s d_e} \to \R^{d_x}$ to the concatenated token embeddings 
to obtain the patch embeddings of $\bmu\ps{l}_k$. Namely, we define 
\begin{equation}
  \label{eq: learner: x}
  \x_{\bmu_k\ps{l}}
  := \bPhi\ps{l}\big( \h\ps{l}_{(k-1)s + 1} \circ \cdots \circ \h\ps{l}_{ks} \big) 
  \in \R^{d_x} , \quad 
  \forall k \in \big[s^{l-1}\big], 
\end{equation}
where $\circ$ means vector concatenation. We call $\{ \x_{\bmu_k\ps{l}} \}_{k \in [s^{l-1}]}$ the patch embeddings. 
In our final optimization result, we will choose $\bPhi\ps{l}$ to be the random Fourier features associated with the radial 
basis function (RBF) kernel (see \eqref{eq: RBF random feature layer} for the exact definitions); however, other 
nonlinear functions can also be used (cf.~Section~\ref{sec: warm up approx}).  

Let $\W\ps{l} \in \R^{d_y \times d_x}$ denote the trainable parameter of the level-$l$ filter. We apply it to each 
of the patch embeddings to form the (unnormalized) outputs  $\W\ps{l} \x_{\bmu_k\ps{l}} \in \R^{d_e}$. For reasons
that will become clear later, we normalize the outputs to be probability vectors (with potential negative
entries):
\begin{equation}
  \label{eq: learner: output}
  \h\ps{l-1}_k :=  \W\ps{l} \x_{\bmu_k\ps{l}} \,/\, \big\langle \One, \W\ps{l} \x_{\bmu_k\ps{l}} \big\rangle \in \R^{d_e} , 
  \quad \forall k \in \big[s^{l-1}\big].
\end{equation}
Note that these are also the inputs of level $l-1$ and the normalization step can always be absorbed into the nonlinear
function $\bPhi\ps{l-1}$. 
The final output of our network is the level-$1$ output $\h\ps{0}_1 \in \R^{d_e}$. We will choose $d_e = V$ and train our 
network to predict $\e_{\zeta}$, where $\zeta \in [V] \cong \cV_0$ is the label. Hence, we can set 
$\argmax_{i \in [V]} [ \y\ps{1}_1 ]_i$ to be the prediction of our model. 

We will show that, after training, $\W\ps{l}$ approximately maps each $\x_{\bmu}$ to conditional probability vector
of the label given the first level-$l$ patch being $\bmu$, and this probability vector can be used as the proxy for 
the unique level-$(l-1)$ symbol that generates $\bmu$. 
See Section~\ref{sec: warm up approx} and Section~\ref{sec: opt results and proof sketch} for details.

\subsection{Training algorithm}
\label{subsec: training algorithm}

We train our learner network \eqref{eq: learner: x}--\eqref{eq: learner: output} layerwise using (empirical)
gradient descent. The algorithm proceeds in $L$ stages, from stage $L$ down to stage $1$. In stage $l \in [L]$,
we train $\W\ps{l}$ as follows. 

Fix $l \in [L]$. We first sample $N\ps{l}$ i.i.d.~samples of form $(\mu_1, \dots, \mu_{s^{L}}, \zeta) \in \cV_L^{s^L}
\times \cV_0$ from the RHM, where $\mu_1, \dots, \mu_{s^L}$ are the (bottom-level) input tokens and $\zeta$ is the label. 
For notational simplicity, let $\hat\E\ps{l}$ denote the empirical expectation w.r.t.~these $N\ps{l}$ samples, 
so that we do not have to explicitly write down the dataset in the future. 
Let $\x_{\smash{ \bmu\ps{l}_1} }, \W\ps{l}$ be defined as in \eqref{eq: learner: x}--\eqref{eq: learner: output}. Our loss is 
\begin{equation}
  \label{eq: training loss}
  \Loss\ps{l}(\W\ps{l})
  := \frac{1}{2} \hat{\E}\ps{l}\Big\| \e_\zeta - \W\ps{l}\x_{\bmu\ps{l}_1} \Big\|^2 + \frac{\lambda_W}{2} \norm{\W\ps{l}}_F^2 ,
\end{equation}
where $\lambda_W \ge 0$ is a hyperparameter controlling the regularization strength. 
Namely, we train the $l$th level of the network to predict the label based on the \emph{first} patch at level $l$. 
We make two remarks on the training procedure. 

First, at each stage, we use only the first patch of the corresponding level. This simplifies the analysis
and suffices for learning the production rules thanks to the patch-level shift invariance of RHMs. 
In addition, since level $1$ has only one patch and all subsequent levels are generated from that patch, 
when we reach level $1$, the model will be able to see all the information contained in the input sequence, so,
in principle, we lose no information by restricting the model to the first patch in earlier stages.

Second, we ignore all the subsequent layers of the learner model when training level $l$, as those untrained layers 
will only obscure the signals from the label. 
Another way to interpret this design choice is pretending there is a residual link 
(\cite{he_deep_2016}) from the $l$th level to the output during the training of the $l$th level, so that the $l$th level 
can directly receive signal from the label. 

We train each level by running gradient descent for $T\ps{l} \in \bbN$ steps. Namely, we update the weights using 
$\W\ps{l}_{t+1} = \W\ps{l}_t - \eta\ps{l} \nabla_{\W} \Loss\ps{l}(\W\ps{l}_t)$, where $\W\ps{l}_0 = 0$
and  $\eta\ps{l} > 0$ is the step size. 
Though the optimization problem as a whole is non-convex, for each 
fixed level $l$, \eqref{eq: training loss} is a standard ridge regression task, which is (strongly) convex and can be 
easily solved by gradient descent. One may as well minimize \eqref{eq: training loss} directly by computing the 
pseudo-inverse of the covariance matrix. Our analysis is rather flexible w.r.t.~the choice of the optimization algorithm.

\section{Warm-up: Approximation results}
\label{sec: warm up approx}

As a warm-up, in this section, we consider the problem of approximating the RHM using (variants of) our learner model. 
Throughout this section, we assume that the RHM is non-ambiguous and $(V, m)$-uniform. 

\paragraph{Construction 1: assuming access to the intermediate nodes.}
First, consider the bottom level of the network. At this level, the inputs are the one-hot encoding vectors 
$\e_{\mu_1\ps{L}}, \dots, \e_{\mu_{s^L}\ps{L}}$ of the level-$L$ tokens. In particular, this means we can easily
obtain orthonormal (equivalently, one-hot up to a rotation) embeddings for the patches. One robust but slightly wasteful approach 
is to choose the nonlinear function to be $\h_1 \circ \cdots \circ \h_s \mapsto \h_1 \otimes \cdots \otimes \h_s$. 
With this nonlinear function $\bPhi$, we have  $\inprod{\x_{\bmu}}{\x_{\bmu'}} = \indi\{ \bmu = \bmu' \}$. 

Recall that $\cP_L \subset \cV_L^s$ denotes the collection of possible patches at level $L$.
Since our RHM is assumed to be non-ambiguous, each patch $\bmu \in \cP_L$ corresponds to a 
unique level-$(L-1)$ token $\nu(\bmu) \in \cV_{L-1}$. Let $\{ \e_\nu \}_{\nu \in \cV_{L-1}}$ denote the one-hot 
encoding vectors of the vocabulary $\cV_{L-1}$. We can then set $\W\ps{L}
= \sum_{\bmu \in \cP_L } \e_{\nu(\bmu)} \x_{\bmu}\trans$,
so that, by the orthonormality of the patch embeddings, $\W\ps{L} \x_{\bmu_k} = \e_{\nu(\bmu_k)} = \e_{\mu\ps{L-1}_k}$. With this $\W\ps{L}$, the inputs of $(L-1)$-th level 
of the network will be exactly the one-hot encoding vectors of the level-$(L-1)$ sequence. We can then repeat the 
same construction, so that eventually in the top level, $\W\ps{1}$ maps the level-$1$ patches to the 
unique labels that generate them. 

\paragraph{Construction 2: conditional probabilities as the surrogate.}
It is unreasonable to hope that a trained model can recover the one-hot encoding scheme $\{ \e_{\nu} \}_{\nu \in \cV_{L-1}}$
in the previous construction without access to the intermediate sequences, and, fortunately, it does not have to. 
It suffices to map the patches $\bmu \in \cP_l$ to any 
embeddings satisfying (i) $\h\ps{l-1}_{\bmu} \approx \h\ps{l-1}_{\bmu'}$ if $\bmu, \bmu'$ are synonyms, and 
(ii) $\h\ps{l-1}_{\bmu}, \h\ps{l-1}_{\bmu'}$ are reasonably different when $\bmu, \bmu'$ are not synonyms. 
After that, it suffices to choose the nonlinear function $\bPhi\ps{l}$ to be some clustering algorithm that 
performs clustering on the level-$l$ outputs, and maps level-$l$ outputs to the one-hot embedding vector of the 
corresponding cluster. 

It was also observed in \cite{cagnetta_how_2024} that, under mild conditions, synonym detection can be reduced to estimating 
the conditional probability of the label given the patch. Formally, for each $l \in [L]$ and $\bmu \in \cP_l$,
we define 
\begin{equation}
  \label{eq: cond. prob. vec. q l mu}
  \q\ps{l}_{\bmu}
  := \big(
      \P\left[ \text{label is $\zeta$} \mid \text{the first patch at level $l$ is $\bmu$} \right]
    \big)_{\zeta \in \cV_0}
  \in \Delta_{\cV_0},
\end{equation}
where $\Delta_{\cV_0}$ denotes the probability simplex over $\cV_0$. Clear that if $\bmu, \bmu'$
are synonyms, then $\q\ps{l}_{\bmu} = \q\ps{l}_{\bmu'}$. Moreover, one can expect that, unless the RHM instance is 
generated adversarially, $\q\ps{l}_{\bmu}, \q\ps{l}_{\bmu'}$ will be different if $\bmu, \bmu'$ are not synonyms. 
In other words, condition~(i) of Informal~Principle~\ref{principle: shallow to deep chaining} holds. 
In \cite{cagnetta_how_2024}, the authors argue that the difference should be asymptotically proportional to $m^{-l/2}$
if the production rules are sampled uniformly at random. Their argument is based on some variance calculation and 
relies on various asymptotic approximations. We prove a rigorous, quantitative version of this claim for production rules 
that are randomly sampled while maintaining a certain marginal uniformity condition. See Proposition~\ref{prop: signal 
lower bound} and Appendix~\ref{sec: signal lower bounds}. We remark that the conditional probability vector 
$\q\ps{l}_{\bmu}$ is a natural quantity to estimate when the task is to predict the label, and also naturally arises from 
minimizing the ridge regression loss \eqref{eq: training loss} (see also \eqref{eq: hat E e xT}).

Thanks to the distinguishing property of $\q\ps{l}_{\bmu}$, we can set the trainable weights of level $l$ to be 
\(
  \W\ps{l} 
  = m\inv \sum_{\bmu \in \cP_L} \q\ps{l}_{\bmu} \x_{\bmu}\trans.
\)
By the orthonormality of $\{\x_{\bmu}\}_{\bmu}$, this maps each $\bmu$ to the corresponding $\q\ps{l}_{\bmu}$, which 
satisfies $\q\ps{l}_{\bmu} \approx \q\ps{l}_{\bmu'}$ if and only if $\bmu, \bmu'$ are synonyms.

\paragraph{Construction 3: replacing clustering with RBF random features.}
In the above construction, we choose $\bPhi$ to be a clustering algorithm to form orthonormal patch embeddings from the 
token embeddings, assuming that the token embeddings respect the synonym structure. 
We claim that this mapping can be replaced by a random feature mapping associated with the RBF kernel, which is a much more 
commonly used and generic component. Though the embeddings obtained using the RBF random features are only 
near orthonormal, one can still expect that a construction similar to Construction 2 can still approximately recover the 
label. Moreover, we show that this construction can be efficiently learned by layerwise training our learner model. 
See Section~\ref{sec: opt results and proof sketch} for the formal description and a proof sketch, and 
see Appendix~\ref{sec: opt proofs} for the proof.

\section{Optimization results and proof sketch}
\label{sec: opt results and proof sketch}

In this section, we sketch the proof of the end-to-end optimization guarantee corresponding to the approximation result (Construction 3) in Section~\ref{sec: warm up approx}.
The full proofs of the results in this section can be found in Appendix~\ref{sec: opt proofs}.
We emphasize that we do not assume access to the intermediate tokens. Instead, we show that the intermediate features 
will be automatically learned when the model is trained to fit the label.

As mentioned in the introduction, our proof consists of two parts: First, we will consider 
each single layer and show that certain properties --- in our case, clusters respecting the synonyms structure --- can 
be approximately preserved across layers. Then, we will chain those one-layer results, show that the error does not 
blow up during chaining, and derive our main optimization guarantees for learning RHMs with a deep convolutional network. 

The following is our main assumption on the RHM instance. See Definition~\ref{def: rhm}-\ref{def: cfg rhm} and 
the remarks following them for the basic definitions related to RHMs. 
\begin{assumption}[Assumptions on the RHM]
  \label{assumption: rhm}
  Let $L, s \in \bbN$ be parameters. Consider the $L$-level RHM $ ( (\cV_l)_{l\in[L]_0}, \cD_0, (\cQ_l)_{l \in [L-1]_0} )$
  with branching factor $s$ that is induced by a CFG, and is $(V, m)$-uniform and non-ambiguous.
  We assume the following are true. 
  \begin{enumerate}[(a)]
    \item \label{itm: rhm non-degeneracy}
      \tnbf{(Non-degeneracy)} For any $l \in [L]$ and $\bmu \in \cP_l$, $\P[ \text{the first patch at level $l$ is $\bmu$} ] 
      \ge ( \kappa |\cP_l|)\inv$.
    \item \label{itm: rhm nontrivial signals}
      \tnbf{(Nontrivial signal)} For each level $l \in [L]$ and level-$l$ non-synonyms $\bmu, \bmu' \in \cP_l$, 
      \begin{equation}
        \label{eq: identifiability, dist >= rho l}
        \smash{ 
        \big\| \q\ps{l}_{\bmu} - \q\ps{l}_{\bmu'} \big\| \ge K_\rho m^{-l/2} =: \rho\ps{l},
        }
      \end{equation}
      where $\q\ps{l}_{\bmu}$ is defined in \eqref{eq: cond. prob. vec. q l mu} and $K_\rho$ is an $l$-independent 
      positive real number that is at least inverse polynomial in $m, V, L, s, \kappa$ and $d^{o(1)}$.
  \end{enumerate}
\end{assumption}

Note that this assumption is deterministic in the sense that it does not rely on a specific production rules sampling procedure. 
We will show that it holds with high probability if the production rules are sampled without replacement uniformly at 
random while maintaining the uniformity of the first token. In general, one could hope this condition to hold, as long 
as the production rules are not generated adversarially. In \eqref{eq: identifiability, dist >= rho l}, the choice of 
$\rho\ps{l} := K_\rho m^{-l/2}$ is tailored to the case where the production rules are 
randomly sampled. Our analysis can be easily extended to situations where $\rho\ps{l}$ has a different form and yields 
better bounds if $\rho\ps{l}$ is much larger than $K_\rho m^{-l/2}$. 

We verify Assumption~\ref{assumption: rhm} in Appendix~\ref{sec: signal lower bounds} in the case where 
$L \sim \log d/\log\log d$ and $m, V, s \sim \log d$, where $d$ is the 
length of the input sequence, and the production rules are randomly sampled with the marginal distribution of the first token 
at each level being kept uniform. See Proposition~\ref{prop: signal lower bound}. Combining Proposition~\ref{prop: signal 
lower bound} and the following theorem gives an end-to-end optimization guarantee.

\begin{restatable}{rTheorem}{OptMain}
  \label{thm: opt main}
  Assume the RHM instance satisfies Assumption~\ref{assumption: rhm}. 
  Then, we can train a learner model~\eqref{eq: learner: x}--\eqref{eq: learner: output} with the RBF random features
  \eqref{eq: RBF random feature layer} using the layerwise gradient descent described 
  in Section~\ref{subsec: training algorithm} to learn this RHM instance with probability at least 
  $1 - \delta_{\P}$.\footnote{
    By the non-ambiguity condition, the label $\zeta$ can be uniquely determined by 
    the input $\x$. We say the model learns the RHM instance if it outputs $\e_\zeta \pm_2 o(1)$, so that we can recover
    the label by taking $\argmax$ on the output.
  }
  Moreover, the number of samples we need is at most $\poly(\kappa, m, V, s, L, \log(1/\delta_{\P})) K_\rho^{-2} m^L$, 
  and the model width and number of gradient descent steps are both bounded by  
  $\poly(K_\rho\inv, \kappa, V, s, L, \log(1/\delta_{\P}), m^L) $.
\end{restatable}
\begin{remark}
  Here, $\poly(z)$ means a quantity that is bounded by $|z|^C$ for some universal constant $C$.
  In the regime considered in Proposition~\ref{prop: signal lower bound},  
  $\poly(K_\rho\inv, \kappa, m, V, s, L) = d^{o(1)}$, so the sample complexity is dominated by the 
  factor $m^L = \poly d$. 
  
  With those $d^{o(1)}$ terms ignored, our result matches the heuristic bound for deep networks in \cite{cagnetta_how_2024}.
  In contrast, the heuristic and empirically verified bound for shallow models in \cite{cagnetta_how_2024} is 
  $\Omega(m^{s^L})$, which is exponentially larger than $m^L = \poly(d)$. 
\end{remark}

\subsection{Training a single layer}
\label{subsec: trianing a single layer}

Fix an arbitrary level $l \in [L]$. In this subsection, we consider the training of the level-$l$ weights $\W\ps{l}$.
Our goal is to show that, when the level-$l$ inputs form clusters respecting the synonym structure
(cf.~Assumption~\ref{assumption: token embeddings}), the $l$th level of the model can approximately learn the mapping 
$\bmu \mapsto \q\ps{l}_{\bmu}$, as a result of which, its outputs will have the same cluster structure, up to potentially 
different parameters. 

For notational convenience, we will drop all superscript $l$ throughout this subsection. 
We will use $( \hat\zeta, \hat\nu, \hat\bmu ) \in \cV_0 \times \cV_{l-1} \times \cV_l^s$ to denote the random 
variables representing the label, the first level-$(l-1)$ token, and the first level-$l$ patch and reserve the unhatted 
versions of them for a generic label, level-$(l-1)$ token, and level-$l$ patch, respectively. 
For each $\mu \in \cV_l$, let $\hat{\h}_\mu \in \R^{d_e}$ denote the random vector representing the token embedding
of $\mu$. When $l = L$, $\hat{\h}_{\mu}$ is the (deterministic) one-hot encoding vector of $\mu$. When $l < L$, 
$\hat{\h}_\mu$ is the output of the $(l+1)$th level of the network at the corresponding patch. $\hat{\h}_\mu$ is random,
as it depends on the values taken in the subtree of $\mu$ (cf.~Figure~\ref{fig: rhm tree}). In addition, since 
we fix all other layers when training the $l$th level, this is the only source of randomness. 
For notational simplicity, for each $\mu \in \cV_l$, let $\tilde{\cP}_{\rmR, l, \mu}$ denote the support of 
$\hat{\h}_{\mu}$.\footnote{Here, the subscript R stands for ``representation''.} 
Finally, for $\hat\bmu = (\hat\mu_1, \dots, \hat\mu_s)$, note that $\hat{\h}_{\hat\mu_i}$ is conditionally independent
of all other $\{ \hat{\h}_{\hat{\mu}_j} \}_{j \ne i}$ and the label $\hat\zeta$ when conditioned on the value of 
$\hat\mu_i$. 

Recall that at the end of the day, we will chain the one-layer results, and to do so, we need to ensure certain 
properties hold across layers (with potentially different parameters). In our case, we will maintain the following
induction hypothesis. Note that at level $L$ (where the token embeddings are one-hot embeddings of $\cV_L$), 
this assumption holds with $\tilde{\eps}_S = 0$ and 
$\tilde{\rho} = \sqrt{2}$.
\begin{assumption}[Token embeddings]
  \label{assumption: token embeddings}
  There exists some $0 \le \tilde{\eps}_S < \tilde{\rho}$ such that the following hold: 
  \begin{enumerate}[(a)]
    \item \label{itm: token intracluster distance}
      \tnbf{(Intracluster)}
      $\norm{ \h_\mu - \h_\mu' } \le \tilde{\eps}_S$ for any $\mu \in \cV_l$ and $\h_\mu, \h_\mu' \in \tilde{\cP}_{\rmR, l, \mu}$.
    \item \label{itm: token intercluster distance}
      \tnbf{(Intercluster)}
      $\norm{\h_\mu - \h_{\mu'} } \ge \tilde{\rho}$ for any $\mu \ne \mu' \in \cV_l$ and 
      $\h_\mu \in \tilde{\cP}_{\rmR, l, \mu}, \h_{\mu'} \in \tilde{\cP}_{\rmR, l, \mu'}$.
  \end{enumerate}
\end{assumption}

\subsubsection{Error propagation in the random features layer}
Let $\h := \h_1 \circ \cdots \circ \h_s \in \R^{s d_e}$ be the concatenation of the input token embeddings (in the 
first patch). 
The nonlinear function $\bPhi: \R^{s d_e} \to \R^{d_x}$ we consider here is the random Fourier features associated with 
the RBF kernel $\psi_\sigma$. That is, 
\begin{equation}
  \label{eq: RBF random feature layer}
  \bPhi(\h)
  := \bPhi_{\sigma, M}(\h) 
  := M^{-1/2} \left( \cos\left( \bomega_k \cdot \h \right), \sin\left( \bomega_k \cdot \h \right) \right)_{k \in [M]},
\end{equation}
where $( \bomega_k )_k$ are i.i.d.~$\Gaussian{0}{\sigma^{-2} \Id_{s d_e}}$ random vectors. Note that $\| \bPhi(\h) \| = 1$
always holds. Moreover, it is well-known that this random feature map approximates the RBF kernel $\psi_\sigma$ uniformly
on compact sets (\cite{rahimi_random_2007}):
\[ 
  \inprod{\bPhi(\h)}{\bPhi(\h')}
  \approx
  \psi_\sigma(\h, \h') 
  := \exp\left( -  \norm{\h - \h'}^2 / (2 \sigma^2) \right)
  = \prod_{i=1}^{s} \exp\left( - \norm{\h_i - \h_i'}^2 / (2 \sigma^2) \right). 
\]
We choose $\bPhi$ to be the RBF random features, because it not only is widely used, but also naturally maps far away 
points to near-orthogonal representations. We conjecture that any random feature mapping that has a sufficiently large 
high-frequency component can work here. Note that, by Assumption~\ref{assumption: token embeddings}\ref{itm: token 
intercluster distance}, if $\h, \h'$ are (pre-$\bPhi$) representations of two distinct patches, then at least one 
$\| \h_i - \h_i' \|$ will be at least $\tilde{\rho}$. Hence, $\psi_\sigma(\h, \h') \approx 0$, as long as $\sigma 
\ll \tilde{\rho}$. Meanwhile, by Assumption~\ref{assumption: token embeddings}\ref{itm: token intracluster distance},
if $\h, \h'$ are (pre-$\bPhi$) representations of the same patch, then $\psi_\sigma(\h, \h') \approx 1$, provided that 
$\tilde{\eps}_S \ll \sigma$. Formally, we have the following lemma, whose proof can be found in Appendix~\ref{subsec: 
appendix: RF}.

\begin{restatable}[Error propagation in the random features layer]{rLemma}{ErrorPropRF}
  \label{lemma: error propagation in the RF layer}
  Let $\delta_{\P} \in (0, 1)$ be the target failure probability and $\eps_S, \eps_O$ the target patch-level 
  intracluster distance and near-orthogonality error. Suppose that Assumption~\ref{assumption: token embeddings} 
  holds and 
  \[
    \sigma^2 \le \frac{2 \tilde{\rho}^2}{\log( 2 / \eps_O )}, \quad 
    M \gtrsim \frac{s d_e}{ \eps_O^2 \wedge \eps_S^4 } 
      \log\left( 
        \frac{
          s d_e \log( 2 / \eps_O )
        }{ 
          \tilde{\rho}^2 
          ( \eps_O^2 \wedge \eps_S^4 )
          \delta_{\P} 
        } 
      \right) , \quad 
    \tilde{\eps}_S 
    \le \frac{\eps_S \sigma }{\sqrt{s}}
  \]
  Then, with probability at least $1 - \delta_{\P}$, we have 
  \begin{equation}
    \label{eq: Phi(h), eps}
    \abs{ \inprod{\bPhi(\h_{\bmu})}{\bPhi(\h_{\bmu'})} } \le \eps_O, \quad 
    \norm{ \bPhi(\h_{\bmu}) - \bPhi(\h_{\bmu}') } \le \eps_S, 
  \end{equation}
  where $\bmu \ne \bmu' \in \cP_l$ and $\h_{\bmu}, \h_{\bmu}'$ are arbitrary (pre-$\bPhi$) representations of $\bmu$ and 
  $\h_{\bmu'}$ is an arbitrary (pre-$\bPhi$) representation of $\bmu'$, all of them obtained by concatenating the corresponding
  token embeddings. 
\end{restatable}

\subsubsection{Training $\mathbf{W}$}

Recall from \eqref{eq: training loss} that $\W := \W\ps{l}$ is trained to predict the label using the first level-$l$ 
patch. Namely, our optimization task is 
minimizing $\W \mapsto \frac{1}{2} \hat\E \| \e_{\hat\zeta} - \W\hat\x_{\hat\bmu} \|^2 + \frac{\lambda_W}{2} \norm{\W}_F^2$,
where $\hat\E$ represents the expectation over the empirical distribution of the training set and $\hat{\x}_{\hat\bmu}
:= \bPhi(\hat{\h}_{\hat\bmu})$ is the embedding of the first level-$l$ patch (cf.~\eqref{eq: learner: x} and 
Figure~\ref{fig: level-l of the learner model}). For notational simplicity, for each patch $\bmu \in \cP_l$, we write 
$\cP_{\rmR, l, \bmu}$ for the support of $\hat{\x}_{\bmu}$.

This objective is strongly convex and can be easily optimized using, say, gradient descent. 
See Lemma~\ref{lemma: convergence of gradient descent} for the convergence rate. 
In addition, by setting the gradient to zero, one can easily verify that the (unique) minimizer is given by 
\(
  \W_{N, *}
  = \hat\E[ \e_{\hat\zeta} \hat{\x}_{\hat\bmu}\trans ]
    \big(
      \hat\E[ \hat{\x}_{\hat\bmu} \hat{\x}_{\hat\bmu}\trans ]
      + \lambda_W \Id_{d_{\x}}
    \big)\inv,
\)
where $N$ denotes the number of samples we use in this stage. For notational simplicity, for each $\bmu' \in \cP_l$, we 
write $\hat{p}_{\bmu'} := \hat\E[ \indi\{ \hat{\bmu} = \bmu' \} ]$ for the ratio of samples in the dataset with the first 
patch being $\bmu'$. 

Note that $\hat\x_{\hat\bmu}$ and $\e_{\hat{\zeta}}$ are conditionally independent when conditioned 
on $\hat\bmu$. Hence, for the first factor in $\W_{N, *}$, we have 
\begin{equation}
  \label{eq: hat E e xT}
  \textstyle
  \hat\E\left[ \e_{\hat\zeta} \hat{\x}_{\hat\bmu}\trans \right]
  = \sum_{\bmu \in \cP_l} \hat{p}_{\bmu} \hat\E\left[ \e_{\hat\zeta} \hat{\x}_{\hat\bmu}\trans \mid \hat\bmu = \bmu \right] 
  = \sum_{\bmu \in \cP_l} \hat{p}_{\bmu} 
    \hat\q_{\bmu}
    \hat\E\left[ \hat{\x}_{\bmu} \right]\trans, 
\end{equation}
where $\hat\q_{\bmu} := \hat\E[ \e_{\hat\zeta} \mid \hat\bmu = \bmu ]$ is exactly the empirical version of $\q\ps{l}_{\bmu}$ (defined in 
\eqref{eq: cond. prob. vec. q l mu}). In addition, recall from Section~\ref{sec: warm up approx} that, given 
Assumption~\ref{assumption: rhm}\ref{itm: rhm nontrivial signals}, $\q\ps{l}_{\bmu}$ can be used to 
identify synonyms. At a conceptual level, these two properties correspond to condition (ii) and (iii) of 
Informal Principle~\ref{principle: shallow to deep chaining}, respectively. 

Now, consider the covariance matrix term in $\W_{N, *}$. For ease of presentation, assume that $\eps_O = 0$, so 
that the embeddings of different patches lie in orthogonal subspaces. Let $\Proj_{\bmu} \in \R^{d_x \times d_x}$
denote the projection matrix onto the subspace spanned by the representations of patch $\bmu$ and $\Proj_\perp 
\in \R^{d_x \times d_x}$ the projection matrix onto the subspace orthogonal to all patch embeddings. 
Then, we have 
\[
  \textstyle
  \left(
    \hat\E\left[ \hat{\x}_{\hat\bmu} \hat{\x}_{\hat\bmu}\trans \right]
    + \lambda_W \Id_{d_{\x}}
  \right)\inv
  = \sum_{\bmu \in \cP_l} \left(
        \hat{p}_{\bmu} \hat\E\left[ \hat{\x}_{\bmu} \hat{\x}_{\bmu}\trans \right]
        + \lambda_W \Proj_{\bmu}
      \right)^\dagger
      + \lambda_W\inv \Proj_\perp, 
\]
In particular, this, together with the orthogonality assumption ($\eps_O = 0$), imply
\begin{equation}
  \label{eq: W N* x, epsO=0}
  \W_{N, *} \x_{\bmu}
  = \hat{p}_{\bmu} 
    \left(
      \hat\E\left[ \hat{\x}_{\bmu} \right]\trans
      \left( \hat{p}_{\bmu} \hat\E\left[ \hat{\x}_{\bmu} \hat{\x}_{\bmu}\trans \right] + \lambda_W \Proj_{\bmu} \right)^\dagger
      \x_{\bmu}
    \right)
    \hat\q_{\bmu},
  \quad\forall \bmu \in \cP_l, \, \x_{\bmu} \in \cP_{\rmR, l, \bmu}.
\end{equation}
When $\eps_O \ne 0$, we show in 
Lemma~\ref{lemma: empirical solution (factoring out epsO)} that \eqref{eq: W N* x, epsO=0} still approximately holds,
and the error scales with $|\cP_l|^2 \eps_O$. The proof is based on the dual formula of the ridge regression solution, 
which allows us to exploit the near-orthogonality of the representations without touching the intracluster structure.

Note that $\hat{\q}_{\bmu}$ is a probability vector and the term in the parentheses of \eqref{eq: W N* x, epsO=0}
is a scalar. Thus, in principle, we can recover $\hat{\q}_{\bmu}$ by normalizing $\W_{N, *}\x_{\bmu}$ with 
$\One\trans\W_{N, *}\x_{\bmu}$. To this end, we need a lower bound on the term in the parentheses. Recall that 
$\| \x_{\bmu} \| = 1$ always holds. If $\eps_S = 0$, i.e., $\cP_{\rmR, l, \bmu} = \{ \x_{\bmu} \}$, then we have 
$\Proj_{\bmu} = \x_{\bmu}\x_{\bmu}\trans$ and that term in the parentheses is equal to 
\[
  \x_{\bmu}\trans
  \left( (\hat{p}_{\bmu} + \lambda_W ) \x_{\bmu} \x_{\bmu}\trans \right)^\dagger
  \x_{\bmu}
  = \x_{\bmu}\trans (\hat{p}_{\bmu} + \lambda_W )\inv \x_{\bmu} \x_{\bmu}\trans \x_{\bmu}
  = (\hat{p}_{\bmu} + \lambda_W )\inv, 
\]
for which we have $\hat{p}_{\bmu} / (\hat{p}_{\bmu} + \lambda_W ) \gtrsim  1/\kappa$ if we choose $\lambda_W = |\cP_l|\inv$. 
When $\eps_S \ne 0$, one can factor out the errors using the standard expectation-variance decomposition of the second 
moment and the Woodbury matrix identity. We show in Lemma~\ref{lemma: q coefficient lower bound} that 
when $\eps_S \le \lambda_W/4$, the term in the parentheses can be lower bounded by $(\hat{p}_{\bmu} + \lambda_W )\inv/2$.

Combining the above two estimates, we show that $\W_{N, *}\x_{\bmu} / \One\trans\W_{N, *}\x_{\bmu} 
\approx \hat{\q}_{\bmu} \approx \q_{\bmu}$ when $N$ is large, 
where the second approximation comes from standard concentration and the fact
that $\hat{\q}_{\bmu}$ is the empirical version of $\q_{\bmu}$.
By Assumption~\ref{assumption: rhm}\ref{itm: rhm nontrivial signals}, this implies that, when $N$ is sufficiently large,
we can ensure Assumption~\ref{assumption: token embeddings} still holds (with a different set of parameters) at the next level. 
Formally, we have the following lemma, which we prove in Appendix~\ref{subsec: appendix: chaining}. 
In particular, note that the dependence of $\eps\ps{l+1}_*$ on $\eps\ps{l}_*$ is logarithmic, so the accuracy requirements
will not blow up at the bottom levels. 

\begin{restatable}[Training a single layer]{rLemma}{TrainingASingleLayer}
  \label{lemma: training a single layer}
  Let $\eps\ps{l}_* > 0$ denote the target accuracy at level $l$, 
  and the level-$l$ weights $\W\ps{l} \in \R^{d_y \times M\ps{l}}$ be obtained by running gradient descent on \eqref{eq: training loss} with $N\ps{l}$ samples for $T\ps{l}$ steps.
  Suppose that level-$(l+1)$ accuracy $\eps\ps{l+1}_*$ and the number of samples $N\ps{l}$ satisfy
  \begin{equation}
    \label{eq: conditions on eps l* and N l}
    \eps\ps{l+1}_*
    \le \frac{ \rho\ps{l+1} }{\sqrt{ 2000 V^2 m^2 s \log( V m \kappa / \eps\ps{l}_* )}}, \quad 
    N\ps{l} \ge 20 \left( \kappa \vee (\eps\ps{l}_*)^{-2} \right) \log(V m /\delta_{\P}), 
  \end{equation}
  and $M\ps{l}, T\ps{l}$ are some large universal polynomials
  of the RHM parameters, (logarithm of) failure probability, and accuracies $\eps\ps{l}_*, \eps\ps{l+1}_*$.
  Then, with probability at least $1 - O(\delta_{\P})$, we have 
  \[
    \W\ps{l}\x_{\bmu} \big/ \langle \One, \W\ps{l}\x_{\bmu} \rangle
    = \q_{\bmu}\ps{l} \pm_2 \eps\ps{l}_*, \quad 
    \forall \bmu \in \cP_l, \, \x_{\bmu} \in \cP_{\rmR, l, \bmu}.
  \]
\end{restatable}

\subsection{Shallow-to-deep chaining}

To prove our main theorem (Theorem~\ref{thm: opt main}), it suffices to chain Lemma~\ref{lemma: training a single layer}.
By Lemma~\ref{lemma: training a single layer}, we only need to choose the target accuracies $\{ \eps\ps{l}_* \}_l$
and the number of samples so that \eqref{eq: conditions on eps l* and N l} holds. To this end, to achieve $o(1)$ target 
accuracy at the last level, it suffices to choose 
\[
  \eps\ps{l}_*
  = \frac{ K_\rho m^{-l/2} }{ \sqrt{ C V^2 m^2 s  L\log m }} 
  \quad\text{and, consequently,}\quad
  N\ps{l} = C \kappa V^2 m^2 s L\log^2(V m \kappa /\delta_{\P}) K_\rho^{-2} m^l,
\]
for some sufficiently large universal constant $C > 0$.
Note that $N\ps{l}$ is a geometric sequence. Thus, 
the total number of needed samples satisfies $\sum_{l=1}^{L} N\ps{l} \lesssim 
\kappa V^2 m^2 s L\log^2(V m \kappa /\delta_{\P}) K_\rho^{-2} m^L$.

\section{Conclusion}
\label{sec: conclusion}

In this work, we study the problem of learning an $L$-level Random Hierarchy Model with $m$ production rules per symbol
using a deep convolutional network. We show that the sample complexity is $O(m^{(1+o(1))L})$, matching the heuristic
bound in \cite{cagnetta_how_2024}. Our proof is based on Informal Principle~\ref{principle: shallow to deep chaining},
which gives conditions under which training a deep network can be reduced to layerwise training of individual layers. 

There are many potential future directions. First, we assume the RHM branching factor $s$ is known, and design our 
convolutional network to process disjoint length-$s$ patches. One interesting question is whether our proof 
can be extended to the setting where the topology of the hierarchy is unknown.
Another interesting future direction is to find other problems to which Informal Principle~\ref{principle: shallow to deep chaining}
applies. We provide one such example in Appendix~\ref{sec: deep quad}. Finally, a related direction is to relax 
condition~(ii) of Informal Principle~\ref{principle: shallow to deep chaining}, and incorporate the backward feature 
correction mechanism (\cite{allen-zhu_backward_2023}) into the shallow-to-deep chaining framework.

\section*{Acknowledgments}

The authors would like to thank Francesco Cagnetta and Zihao Wang for helpful discussion and feedback. 
JDL acknowledges support of the NSF CCF 2002272, NSF IIS 2107304, and NSF CAREER Award 2144994. FK acknowledges funding from the Swiss National Science Foundation grants  OperaGOST (grant number $200021\ 200390$), DSGIANGO (grant number $225837$), and from the Simons Collaboration on the Physics of Learning and Neural Computation via the Simons Foundation grant ($\#1257412$).

\bibliography{reference}

\newpage

\appendix

{
  \hypersetup{linkcolor=black}
  \tableofcontents
}

\section{Proof of Theorem~\ref{thm: opt main}: Learning RHMs}
\label{sec: opt proofs}

In this section, we prove Theorem~\ref{thm: opt main}. Following the proof sketch in Section~\ref{sec: opt results and 
proof sketch}, we will first establish an optimization result for each layer, and then chain them to obtain a proof 
of Theorem~\ref{thm: opt main}. 

\subsection{Error propagation in the random features layer}
\label{subsec: appendix: RF}

As discussed in Section~\ref{sec: opt results and proof sketch}, we first analyze the error propagation in the random 
feature layer. Our goal is to show that, as long as the representations of the level-$l$ symbols form clusters, 
then after the random feature layer, the patch embeddings will be approximately orthonormal. 
Formally, we prove the following lemma from the main text. 
\ErrorPropRF*
\begin{proof}
  Let $\eps_{\RF} > 0$ denote the error in the random feature approximation (cf.~Lemma~\ref{lemma: uniform convergence 
  of the random fourier features}). We now derive conditions on $\eps_{\RF}$ under which \eqref{eq: Phi(h), eps}
  can be realized and conditions on $M$ under which $\eps_{\RF}$ satisfies these conditions with high probability.
  First, consider two distinct patches $\bmu \ne \bmu' \in \cP_l$ and let $\h_{\bmu}, \h_{\bmu'}$ denote the 
  concatenation of the token embeddings of $\bmu$ and $\bmu'$, respectively. Since $\bmu, \bmu'$ differ by at least 
  one token, by Assumption~\ref{assumption: token embeddings}\ref{itm: token intercluster distance}, we have 
  $\norm{ \h_{\bmu} - \h_{\bmu'} } \ge \tilde{\rho}$. Therefore, by Lemma~\ref{lemma: uniform convergence of the 
  random fourier features}, we have 
  \[
    \abs{ \inprod{\bPhi(\h_{\bmu})}{\bPhi(\h_{\bmu'})} }
    \le \exp\left( - \frac{\tilde{\rho}^2}{2 \sigma^2} \right) + \eps_{\RF} .
  \]
  For the RHS to be smaller than $\eps_O$, it suffices to choose
  \[
    \sigma^2 \le \frac{2 \tilde{\rho}^2}{\log( 2 / \eps_O )} , \quad 
    \eps_{\RF} \le \frac{\eps_O}{2} .
  \]
  Then, consider two different representations $\h_{\bmu}, \h_{\bmu}'$ of the same patch $\bmu$. 
  By Assumption~\ref{assumption: token embeddings}\ref{itm: token intracluster distance}, we have 
  $\norm{ \h_{\bmu} -  \h_{\bmu}' } \le \sqrt{s} \tilde{\eps}_S$. Therefore, by Lemma~\ref{lemma: uniform convergence of 
  the random fourier features}, we have 
  \[
    \inprod{\bPhi(\h_{\bmu})}{\bPhi(\h_{\bmu}')}
    \ge \exp\left( - \frac{s \tilde{\eps}_S^2}{2 \sigma^2} \right) - \eps_{\RF}
    \ge 1 - \frac{s \tilde{\eps}_S^2}{2 \sigma^2} - \eps_{\RF}.
  \]
  Since the output of $\bPhi$ always has unit norm, this implies
  \[
    \norm{ \bPhi(\h_{\bmu} ) - \bPhi(\h_{\bmu}' ) }
    = \sqrt{ 2( 1 - \inprod{ \bPhi(\h_{\bmu}) }{ \bPhi(\h_{\bmu}') } ) } 
    \le \sqrt{ \frac{s \tilde{\eps}_S^2}{\sigma^2} + 2 \eps_{\RF} }.
  \]
  For the RHS to be bounded by $\eps_S$, it suffices to require
  \begin{align*}
    \frac{s \tilde{\eps}_S^2}{\sigma^2} + 2 \eps_{\RF} 
    \le \eps_S^2 
    &\quad\Leftarrow\quad 
    \tilde{\eps}_S^2 
    \le \frac{\eps_S^2 \sigma^2}{s}, \quad 
    \eps_{\RF} 
    \le \frac{\eps_S^2}{2}. 
  \end{align*}
  In particular, for $\eps_{\RF}$, it suffices to have $\eps_{\RF} = ( \eps_O \wedge \eps_S^2 ) / 2$. 
  By Lemma~\ref{lemma: uniform convergence of the random fourier features}, to achieve this with probability at 
  least $1 - \delta_{\P}$, it suffices to choose 
  \[
    M 
    \gtrsim \frac{d_h}{\eps_{\RF}^2} \log\left( \frac{d_h}{\sigma^2 \eps_{\RF}^2 \delta_{\P} } \right)
    \gtrsim \frac{s d_e}{ \eps_O^2 \wedge \eps_S^4 } 
      \log\left( 
        \frac{
          s d_e \log( 2 / \eps_O )
        }{ 
          \tilde{\rho}^2 
          ( \eps_O^2 \wedge \eps_S^4 )
          \delta_{\P} 
        } 
      \right) .
  \]
\end{proof}

\begin{lemma}[Claim 1 of \cite{rahimi_random_2007}]
  \label{lemma: uniform convergence of the random fourier features}
  Fix $\sigma > 0$. Let $\psi_\sigma$ and $\bPhi_{\sigma, M}$ be defined as in Section~\ref{sec: opt results and proof sketch}. 
  Then, for any $\eps_{\RF} > 0$, we have 
  \[
    \P\left[ 
      \sup_{ \norm{\h}, \norm{\h'} \le 1 } 
      \abs{ \inprod{\bPhi_{\sigma, M}(\h)}{\bPhi_{\sigma, M}(\h')} - \psi_\sigma(\h, \h') }
      \ge \eps_{\RF}
    \right]
    \le 2^8 \frac{d_h}{\sigma^2 \eps_{\RF}^2} \exp\left( - \frac{M \eps_{\RF}^2}{4 (d_h + 2)} \right). 
  \]
  In other words, to achieve $\eps_{\RF}$ accuracy uniformly over the unit ball with probability at least $1 - \delta_{\P}$,
  it suffices to choose 
  \[
    M \gtrsim \frac{d_h}{\eps_{\RF}^2} \log\left( \frac{d_h}{\sigma^2 \eps_{\RF}^2 \delta_{\P} } \right).
  \]
\end{lemma}

\subsection{Training a single layer}

Fix an intermediate level $l \in [L]$. In this subsection, we consider the training of the level-$l$ weights
$\W\ps{l}$. For notational simplicity, all superscript $l$ will be dropped, as we consider only a single level here. 
Recall that, for each $\bmu' \in \cP_l$, we write $\hat{p}_{\bmu'} := \hat\E[ \indi\{ \hat{\bmu} = \bmu' \} ]$ for the ratio of samples in the dataset with the first 
patch being $\bmu'$, $\hat{p}_{\max} = \max_{\bmu \in \cP_l} \hat{p}_{\bmu}$, and $\hat\p := (\hat{p}_{\bmu})_{\bmu}$. 
We will eventually choose $N$ to be sufficiently large that $\hat{p}_{\bmu} = (1 \pm o(1)) p_{\bmu}$,
where $p_{\bmu} := \E \hat{p}_{\bmu}$. Hence, in the following, we will always assume $\hat{p}_{\bmu} = (1 \pm 0.5) p_{\bmu}$ when 
dealing with the error terms. 
Since we consider regimes where $m, V \lesssim \log d$ and the dominating factor will be $m^L = \poly d$, we will often
use loose inequalities such as $\hat{p}_{\max}, \norm{\hat{\p}} \le 1$ to get cleaner bounds.

For easier reference, we include all needed assumptions on patch embeddings in the following assumption. 
Note that, by Lemma~\ref{lemma: error propagation in the RF layer}, this assumption holds with high probability
when the patch embeddings are constructed using RBF random features.
\begin{assumption}[Patch embeddings]
  \label{assumption: patch}
  We assume the following on the level-$l$ patch embeddings:
  \begin{enumerate}[(a)]
    \item \label{itm: patch unit norm}
      \tnbf{(Unit norm)} $\norm{\x_{\bmu}} = 1$ for any $\bmu \in \cP_l$ and $\x_{\bmu} \in \cP_{\rmR, l, \bmu}$
    \item \label{itm: patch near orthogonality}
      \tnbf{(Near-orthogonality)} There exists a small nonnegative $\eps_O \lesssim 1/|\cP_l|$ such that 
      $| \x_{\bmu} \cdot \x_{\bmu'} | \le \eps_O$, for any $\bmu \ne \bmu' \in \cP_l$
      and $\x_{\bmu} \in \cP_{\rmR, l, \bmu}$, $\x_{\bmu'} \in \cP_{\rmR, l, \bmu'}$.
    \item \label{itm: patch intracluster distance}
      \tnbf{(Intracluster distance)} There exists a small $\eps_S \ge 0$ such that 
      $\norm{ \x_{\bmu} - \x_{\bmu}' } \le \eps_S$ for any $\bmu \in \cP_l$ 
      and $\x_{\bmu}, \x_{\bmu}' \in \cP_{\rmR, l, \bmu}$.
  \end{enumerate}
\end{assumption}

Recall from Section~\ref{subsec: trianing a single layer} that our goal is to show that $\W_{N, T}$ will approximately 
map $\x_{\bmu}$ to (a rescaled version of) $\q_{\bmu}$. In addition, recall that the optimal empirical solution 
is given by 
\[
  \W_{N, *}
  = \hat\E\left[ \e_{\hat\zeta} \hat{\x}_{\hat\bmu}\trans \right]
    \left(
      \hat\E\left[ \hat{\x}_{\hat\bmu} \hat{\x}_{\hat\bmu}\trans \right]
      + \lambda_W \Id_{d_{\x}}
    \right)\inv. 
\]
Our first step is to factor out the error introduced by the non-orthogonality. As shown in Lemma~\ref{lemma: error 
propagation in the RF layer}, it is easy to make the near-orthogonality error $\eps_O$ essentially arbitrarily small, 
so we do not have to strive for sharp estimations.

\begin{lemma}[Factoring out the near-orthogonality error]
  \label{lemma: empirical solution (factoring out epsO)}
  Choose $\lambda_W = 1/|\cP_l|$. 
  Under Assumption~\ref{assumption: patch}, the optimal empirical solution satisfies
  \[
    \W_{N, *}\x_{\bmu}
    = \hat{p}_{\bmu}  
      \left(
        \hat{\E}[ \hat\x_{\bmu} ]\trans
        \left( \hat{p}_{\bmu} \hat\E[ \hat{\x}_{\bmu}\hat{\x}_{\bmu}\trans ] + \lambda_W \Id_{d_x} \right)\inv
        \x_{\bmu}
      \right)
      \hat{\q}_{\bmu}
      \pm_2 6 |\cP_l|^2 \eps_O, 
  \]
  for any $\bmu \in \cP_l$ and $\x_{\bmu} \in \cP_{\rmR, l, \bmu}$.
\end{lemma}
\begin{proof}
  \def\currentprefix{proof: empiricla cov, lka;dmf}
  Let $\hat{\X} \in \R^{d_x \times N}, \hat\bE \in \R^{d_y \times N}$ represent our dataset, where each column 
  of $\hat{\X}$ and $\hat\bE$ corresponds to one $\hat{\x}_{\hat\bmu}$ and $\e_{\hat\zeta}$, respectively. 
  Though not strictly necessary, one should assume we group different representations of $\bmu \in \cP_l$ together, so 
  that $\hat{\X}$ and $\hat\bE$ have a natural block structure. 
  Then, using the push-through identity, we can rewrite $\W_{N, *}$ in the dual form as 
  \begin{equation}
    \locallabel{WN*}
    \W_{N, *}
    = \frac{1}{N} \hat\bE \hat{\X}\trans
      \left(
        \frac{1}{N} \hat{\X}\hat{\X}\trans
        + \lambda_W \Id_{d_{\x}}
      \right)\inv 
    = \hat\bE 
      \left( \hat{\X}\trans\hat{\X} + \lambda_W N \Id_N \right)\inv 
      \hat{\X}\trans.
  \end{equation}
  Let $\G := \hat{\X}\trans\hat{\X}$ denote the Gram matrix. Note that $G_{i, j} = \inprod{\hat{\x}_i}{\hat{\x}_j}$
  for all $i, j \in [N]$, where $\hat{\x}_k$ is the $k$-th column of $\hat{\X}$. 
  We write $i \sim j$, if $\hat{\x}_i, \hat{\x}_j$ are representations of the same input patch $\bmu$. 
  Recall from Assumption~\ref{assumption: patch}\ref{itm: patch near orthogonality} that 
  $\hat{\x}_i, \hat{\x}_j$ are near-orthogonal if $i \not\sim j$. 
  Hence, it is natural to consider $\tilde{\G} := \left[ \inprod{\hat{\x}_i}{\hat{\x}_j} \indi\{i \sim j\}  \right]_{i, j \in [N]}$
  and write $\bDelta := \G - \tilde{\G}$. By Assumption~\ref{assumption: patch}\ref{itm: patch near 
  orthogonality}, the entries of $\bDelta$ are bounded by $\eps_O$. Hence, $\norm{\bDelta}_2 \le N \eps_O$ and, 
  therefore, by the Woodbury matrix identity,
  \begin{align*}
    \left( \G + \lambda_W N \Id_N \right)\inv 
    = \left( \tilde\G + \lambda_W N \Id_N + \bDelta  \right)\inv 
    &= \left( \tilde\G + \lambda_W N \Id_N \right)\inv 
      \pm_2 \frac{\norm{\bDelta}_2}{\lambda_W^2 N^2} \\
    &= \left( \tilde\G + \lambda_W N \Id_N \right)\inv 
      \pm_2 \frac{\eps_O}{\lambda_W^2 N},
  \end{align*}
  where $\A = \mbf{B} \pm_2 \delta$ means $\norm{\A - \mbf{B}}_2 \le \delta$.
  Combining this with (\localref{WN*}) and Assumption~\ref{assumption: patch}\ref{itm: patch unit norm}, we obtain 
  \[
    \W_{N, *}
    = \hat\bE 
      \left( \tilde\G + \lambda_W N \Id_N \right)\inv 
      \hat{\X}\trans
      \pm \hat\bE 
      \left(
        \pm_2 \frac{\eps_O}{\lambda_W^2 N}
      \right)
      \hat{\X}\trans
    =  \hat\bE 
      \left( \tilde\G + \lambda_W N \Id_N \right)\inv 
      \hat{\X}\trans
      \pm_2 \frac{\eps_O}{\lambda_W^2},
  \]
  as the spectral norms of $\hat\bE, \hat{\X}$ are bounded by $\sqrt{N}$.
  Note that, by construction, $\tilde{\G}$ is a block diagonal matrix. 
  For each patch $\bmu \in \cP_l$, let $\tilde{\G}_{\bmu}, \Id_{\bmu} \in \R^{N \times N}, \hat{\X}_{\bmu} \in \R^{d_x \times N}$ 
  denote the block of $\tilde{\G}, \Id_N$ and $\hat{\X}$ corresponding to $\bmu$, respectively. That is, we set the 
  entries corresponding to other $\bmu'$ to $0$. This keeps the shape of the original matrices.
  Then, we can further rewrite the above as 
  \[
    \W_{N, *}
    = \sum_{\bmu', \bmu''} 
      \hat\bE 
      \left( \tilde\G_{\bmu'} + \lambda_W N \Id_{\bmu'} \right)^\dagger
      \hat{\X}_{\bmu''}\trans
      \pm_2 \frac{\eps_O}{\lambda_W^2} 
    = \sum_{\bmu' \in \cP_l} 
      \hat\bE 
      \left( \tilde\G_{\bmu'} + \lambda_W N \Id_{\bmu'} \right)^\dagger
      \hat{\X}_{\bmu'}\trans
      \pm_2 \frac{\eps_O}{\lambda_W^2},
  \]
  where the second equation comes from the fact that the block matrices 
  $\tilde\G_{\bmu'} + \lambda_W N \Id_{\bmu'}$ and $\hat{\X}_{\bmu''}\trans$ act on orthogonal subspaces when $\mu' \ne \mu''$.
  Fix arbitrary $\bmu \in \cP_l$ and $\x_{\bmu} \in \cP_{\rmR, l, \bmu}$ and write 
  \[
    \W_{N, *} \x_{\bmu}
    = \hat\bE 
        \left( \tilde\G_{\bmu} + \lambda_W N \Id_{\bmu} \right)^\dagger
        \hat{\X}_{\bmu}\trans
        \x_{\bmu} 
      + \sum_{\bmu': \bmu' \ne \bmu} 
        \hat\bE 
        \left( \tilde\G_{\bmu'} + \lambda_W N \Id_{\bmu'} \right)^\dagger
        \hat{\X}_{\bmu'}\trans
        \x_{\bmu}
      \pm_2 \frac{\eps_O}{\lambda_W^2}. 
  \]
  For the second term, again by the near-orthogonality, we have 
  \begin{align*}
    \norm{ 
      \sum_{\bmu': \bmu' \ne \bmu} 
        \hat\bE 
        \left( \tilde\G_{\bmu'} + \lambda_W N \Id_{\bmu'} \right)^\dagger
        \hat{\X}_{\bmu'}\trans
        \x_{\bmu}
    }
    &\le \sum_{\bmu': \bmu' \ne \bmu}  \sqrt{N}
      \frac{1}{\lambda_W N} \norm{ \hat{\X}_{\bmu'}\trans \x_{\bmu} } \\
    &\le \frac{1}{\lambda_W \sqrt{N}} \sum_{\bmu': \bmu' \ne \bmu}  \sqrt{N} \eps_O
    \le \frac{|\cP_l| \eps_O}{\lambda_W}.
  \end{align*}
  Therefore, 
  \begin{align*}
    \W_{N, *} \x_{\bmu}
    &= \hat\bE 
        \left( \tilde\G_{\bmu} + \lambda_W N \Id_{\bmu} \right)^\dagger
        \hat{\X}_{\bmu}\trans
        \x_{\bmu} 
      \pm_2 \frac{|\cP_l| \eps_O}{\lambda_W}
      \pm_2 \frac{\eps_O}{\lambda_W^2} \\
    &= \hat\bE 
        \left( \tilde\G_{\bmu} + \lambda_W N \Id_{\bmu} \right)^\dagger
        \hat{\X}_{\bmu}\trans
        \x_{\bmu} 
      \pm_2 3\left( \frac{1}{\lambda_W^2} + |\cP_l|^2 \right) \eps_O.
  \end{align*}
  Finally, note that $\tilde\G_{\bmu'} = \hat{\X}_{\bmu'}\trans\hat{\X}_{\bmu'}$. Hence, we can use the push-through 
  identity on the block to convert the above formula back into the primal form as 
  \begin{align*}
    \W_{N, *}\x_{\bmu}
    &= \hat\bE \hat{\X}_{\bmu}\trans
        \left( \hat{\X}_{\bmu}\hat{\X}_{\bmu}\trans + \lambda_W N \Id_{d_x} \right)\inv
        \x_{\bmu} 
      \pm_2 3\left( \frac{1}{\lambda_W^2} + |\cP_l|^2 \right) \eps_O \\
    &= \frac{1}{N} \hat\bE \hat{\X}_{\bmu}\trans
        \left( \frac{1}{N} \hat{\X}_{\bmu}\hat{\X}_{\bmu}\trans + \lambda_W \Id_{d_x} \right)\inv
        \x_{\bmu} 
      \pm_2 3\left( \frac{1}{\lambda_W^2} + |\cP_l|^2 \right) \eps_O \\      
    &= \hat{p}_{\bmu}  
      \left(
        \hat{\E}[ \hat\x_{\bmu} ]\trans
        \left( \hat{p}_{\bmu} \hat\E[ \hat{\x}_{\bmu}\hat{\x}_{\bmu}\trans ] + \lambda_W \Id_{d_x} \right)\inv
        \x_{\bmu}
      \right)
      \hat{\q}_{\bmu}
      \pm_2 3\left( \frac{1}{\lambda_W^2} + |\cP_l|^2 \right) \eps_O, 
  \end{align*}
  where in the last line, we have used the fact that 
  \[
    \frac{1}{N} \hat\bE\hat{\X}_{\bmu'}\trans 
    = \hat{p}_{\bmu'} \hat{\E}[ \e_{\hat\zeta} \hat{\x}_{\hat\bmu}\trans \mid \hat{\bmu} = \bmu' ]
    = \hat{p}_{\bmu'} \hat{\E}[ \e_{\hat\zeta} \mid \hat{\bmu} = \bmu' ]
      \E[ \hat{\x}_{\bmu'} ]\trans 
    = \hat{p}_{\bmu'} \hat{\q}_{\bmu'} \hat{\E}[ \hat\x_{\bmu'} ]\trans.
  \]
  To complete the proof, recall that we choose $\lambda_W = 1/|\cP_l|$. 
\end{proof}

We now estimate the difference between the optimal empirical solution $\W_{N, *}$ and the gradient descent solution
$\W_{N, T}$. Since the optimization task of each single level is strongly convex, standard analysis yields the following
bound. 

\begin{lemma}[Convergence of gradient descent]
  \label{lemma: convergence of gradient descent}
  Suppose that Assumption~\ref{assumption: patch} holds. 
  Let $\W_{N, T}$ be the gradient descent solution we obtain if we run gradient descent on \eqref{eq: training loss} 
  for $T$ steps from $\W = 0$ with step size $2|\cP_l|/(|\cP_l|+1)$ and $\lambda_W = 1/|\cP_l|$. Then, we have 
  \[
    \norm{ \W_{N, T} - \W_{N, *} }_2 
    \le \norm{ \W_{N, T} - \W_{N, *} }_F 
    \le e^{-T/|\cP_l|} \norm{ \W_{N, *} }_F 
    \le e^{-T/|\cP_l|} |\cP_l|.
  \]
  As a corollary, we have
  \begin{equation}
    \label{eq: WNT x, epsO out}
    \W_{N, T}\x_{\bmu}
    = \hat{p}_{\bmu}  
      \left(
        \hat{\E}[ \hat\x_{\bmu} ]\trans
        \left( \hat{p}_{\bmu} \hat\E[ \hat{\x}_{\bmu}\hat{\x}_{\bmu}\trans ] + \lambda_W \Id_{d_x} \right)\inv
        \x_{\bmu}
      \right)
      \hat{\q}_{\bmu}
      \pm_2 e^{-T/|\cP_l|} |\cP_l|
      \pm_2 6 |\cP_l|^2 \eps_O, 
  \end{equation}
  for any $\bmu \in \cP_l$ and $\x_{\bmu} \in \cP_{\rmR, l, \bmu}$.
\end{lemma}
\begin{proof}
  First, we compute the gradient and Hessian of $\Loss_N$. Let $\bDelta \in \R^{d_{\y} \times d_{\x}}$ be a 
  small perturbation. Then, we compute 
  \begin{align*}
    \Loss_N(\W + \bDelta)
    &= \frac{1}{2} \hat\E \norm{ \e_{\hat\zeta} - \W\hat{\x}_{\hat\bmu} - \bDelta\hat{\x}_{\hat\bmu} }^2
      + \frac{\lambda_W}{2} \norm{\W + \bDelta}_F^2 \\
    &= \frac{1}{2} \hat\E \norm{ \e_{\hat\zeta} - \W\hat{\x}_{\hat\bmu} }^2
      + \frac{1}{2} \hat\E \norm{ \bDelta\hat{\x}_{\hat\bmu} }^2
      - \hat\E \inprod{ \e_{\hat\zeta} - \W\hat{\x}_{\hat\bmu} }{ \bDelta\hat{\x}_{\hat\bmu} }  \\
      &\qquad
      + \frac{\lambda_W}{2} \norm{\W }_F^2
      + \frac{\lambda_W}{2} \norm{ \bDelta}_F^2
      + \lambda_W \inprod{\W}{ \bDelta}  \\
    &= \Loss_N(\W)
      + \inprod{ 
          - \hat\E \left[ \e_{\hat\zeta} \hat{\x}_{\hat\bmu}\trans \right]
          + \W \left( \hat\E \left[ \hat{\x}_{\hat\bmu} \hat{\x}_{\hat\bmu}\trans \right] + \lambda_W \Id_{d_{\x}}  \right)
        }{ \bDelta } 
         \\
      &\qquad
      + \frac{1}{2} 
        \inprod{ 
          \bDelta 
          \left( \hat\E\left[\hat{\x}_{\hat\bmu} \hat{\x}_{\hat\bmu}\trans \right] + \lambda_W \Id_{d_{\x}} \right)
        }{ \bDelta }.
  \end{align*}
  Clear that $\Loss_N$ is $\lambda_W$-strongly convex w.r.t.~the Frobenius inner product. 
  In addition, since $\norm{\hat{\x}_{\hat{\bmu}}}\le 1$ a.s., we have 
  $\hat\E\left[ \hat{\x}_{\hat\bmu} \hat{\x}_{\hat\bmu}\trans \right] \preceq \Id_{d_{\x}}$. 
  In other words, $\Loss_N$ is $1$-smooth w.r.t.~the Frobenius inner product. The first part of lemma then follows 
  directly from Theorem~2.1.15 of \cite{nesterov_introductory_2004} and the na\"ive bound 
  \[
    \norm{\W_{N, *}}_F 
    = \norm{
        \hat\E\left[ \e_{\hat\zeta} \hat{\x}_{\hat\bmu}\trans \right]
        \left(
          \hat\E\left[ \hat{\x}_{\hat\bmu} \hat{\x}_{\hat\bmu}\trans \right]
          + \lambda_W \Id_{d_{\x}}
        \right)\inv 
      }_F \\ 
    \le \lambda_W\inv  \norm{ \hat\E\left[ \e_{\hat\zeta} \hat{\x}_{\hat\bmu}\trans \right] }_F 
    \le \lambda_W\inv 
    = |\cP_l|.
  \]
  The second part of this lemma comes from the first part and Lemma~\ref{lemma: empirical solution (factoring out epsO)}.
\end{proof}

Note that the term in the first parentheses of \eqref{eq: WNT x, epsO out} is a scalar. Hence, in principle, we 
can recover $\hat{\q}_{\bmu}$ up to error scaling with $e^{-T/|\cP_l|}$ and $\eps_O$ by normalizing $\W_{N, T}\x_{\bmu}$ 
with $\inprod{\One}{\W_{N, T}\x_{\bmu}}$. To this end, we need a lower bound on the coefficient of the first term, and 
this is where the small intracluster distances come into play.

\begin{lemma}
  \label{lemma: q coefficient lower bound}
  Suppose that Assumption~\ref{assumption: patch} holds and $\eps_S \le \lambda_W / 4$. Then, we have  
  \[
    \hat{\E}[ \hat\x_{\bmu} ]\trans
      \left( \hat{p}_{\bmu} \hat\E[ \hat{\x}_{\bmu}\hat{\x}_{\bmu}\trans ] + \lambda_W \Id_{d_x} \right)\inv
      \x_{\bmu}
    \ge \frac{1}{4(\hat{p}_{\bmu} + \lambda_W)}, 
  \]
  for any $\bmu \in \cP_l$ and $\x_{\bmu} \in \cP_{\rmR, l, \bmu}$.
\end{lemma}
\begin{proof}
  First, by the standard expectation-variance decomposition of the second moment and 
  Assumption~\ref{assumption: patch}\ref{itm: patch intracluster distance}, we have 
  \begin{align*}
    \left( \hat{p}_{\bmu} \hat\E[ \hat{\x}_{\bmu}\hat{\x}_{\bmu}\trans ] + \lambda_W \Id_{d_x} \right)\inv
    &= \left( 
        \hat{p}_{\bmu} \hat\E[ \hat{\x}_{\bmu} ]^{\otimes 2} 
        + \lambda_W \Id_{d_x} 
        + \hat{p}_{\bmu} \widehat\Var[\hat{\x}_{\bmu}]
      \right)\inv \\
    &= \left( 
        \hat{p}_{\bmu} \hat\E[ \hat{\x}_{\bmu} ]^{\otimes 2} 
        + \lambda_W \Id_{d_x} 
        \pm_2 \hat{p}_{\bmu} \eps_S^2
      \right)\inv \\
    &= \left( \hat{p}_{\bmu} \hat\E[ \hat{\x}_{\bmu} ]^{\otimes 2} + \lambda_W \Id_{d_x}  \right)\inv
      \pm_2 \frac{\hat{p}_{\bmu} \eps_S^2}{\lambda_W^2},
  \end{align*}
  where the last line comes from the Woodbury identity. Then, for the first term, we have 
  \begin{align*}
    \left( \hat{p}_{\bmu} \hat\E[ \hat{\x}_{\bmu} ]^{\otimes 2} + \lambda_W \Id_{d_x}  \right)\inv 
    &= \left( 
        \left( \hat{p}_{\bmu} \norm{ \hat\E[ \hat{\x}_{\bmu} ] }^2 + \lambda_W \right) 
          \overline{\hat\E[ \hat{\x}_{\bmu} ]}^{\otimes 2} 
        + \lambda_W \left( \Id_{d_x} - \overline{\hat\E[ \hat{\x}_{\bmu} ]}^{\otimes 2}   \right)
      \right)\inv \\
    &= \frac{1}{\hat{p}_{\bmu} \norm{ \hat\E[ \hat{\x}_{\bmu} ] }^2 + \lambda_W} 
          \overline{\hat\E[ \hat{\x}_{\bmu} ]}^{\otimes 2} 
        + \frac{1}{\lambda_W} \left( \Id_{d_x} - \overline{\hat\E[ \hat{\x}_{\bmu} ]}^{\otimes 2}   \right).
  \end{align*}
  As a result, we have 
  \[
    \hat{\E}[ \hat\x_{\bmu} ]\trans
      \left( \hat{p}_{\bmu} \hat\E[ \hat{\x}_{\bmu}\hat{\x}_{\bmu}\trans ] + \lambda_W \Id_{d_x} \right)\inv
    = \frac{
        \hat\E[ \hat{\x}_{\bmu} ]\trans 
      }{ \hat{p}_{\bmu} \norm{ \hat\E[ \hat{\x}_{\bmu} ] }^2 + \lambda_W } 
      \pm_2 \frac{\eps_S^2}{\lambda_W^2}.
  \]
  Finally, by Assumption~\ref{assumption: patch}\ref{itm: patch unit norm} and \ref{itm: patch intracluster distance}, 
  as long as $\eps_S \ll 1$, we have 
  \begin{align*}
    \hat{\E}[ \hat\x_{\bmu} ]\trans
    \left( \hat{p}_{\bmu} \hat\E[ \hat{\x}_{\bmu}\hat{\x}_{\bmu}\trans ] + \lambda_W \Id_{d_x} \right)\inv
    \x_{\bmu}
    &= \frac{
        \inprod{\hat\E[ \hat{\x}_{\bmu} ]}{ \x_{\bmu} } 
      }{ \hat{p}_{\bmu} \norm{ \hat\E[ \hat{\x}_{\bmu} ] }^2 + \lambda_W } 
      \pm \frac{\eps_S^2}{\lambda_W^2} \\
    &\ge \frac{ 1/2 }{ \hat{p}_{\bmu} + \lambda_W } \pm \frac{\eps_S^2}{\lambda_W^2}.
  \end{align*}
  For the last term to be lower bounded by $( 4(\hat{p}_{\bmu} + \lambda_W) )\inv$, it suffices to require
  \[
    \frac{\eps_S^2}{\lambda_W^2}
    \le \frac{1}{4(\hat{p}_{\bmu} + \lambda_W)}
    \quad\Leftarrow\quad 
    \eps_S^2 
    \le \frac{\lambda_W^2}{8}
    \quad\Leftarrow\quad 
    \eps_S 
    \le \frac{\lambda_W}{4}.
  \]
\end{proof}

Now, we prove the following main lemma of training a single layer by combining the above two lemmas and choosing 
$N$ sufficiently large that $\hat{\q}_{\bmu}$ and $\hat{p}_{\bmu}$ concentrate around their respective means
(cf.~Lemma~\ref{lemma: concentration of hat q and hat p}).

\begin{lemma}[Training $\W$]
  \label{lemma: train W}
  Suppose that Assumption~\ref{assumption: patch} holds and let $\delta_{\P}, \eps_* \in (0, 0.1)$ 
  be our target failure probability and accuracy, respectively. 
  Let $\W_{N, T}$ be the gradient descent solution we obtain if we run gradient descent on \eqref{eq: training loss} 
  for $T$ steps from $\W = 0$ with step size $2|\cP_l|/(|\cP_l|+1)$ and $\lambda_W = 1/|\cP_l|$. Suppose that 
  \begin{gather*}
    N \ge 8 \left( \kappa \vee \eps^{-2} \right) \log(4|\cP_l|/\delta_{\P}) , \quad 
    T \ge |\cP_l| \log\left( 100 |\cP_l| \kappa  \sqrt{d_y}  / \eps_* \right), \\ 
    \eps_S \le \frac{\lambda_W}{4} = \frac{1}{4 |\cP_l|}, \quad 
    \eps_O \le \frac{\eps_*}{600 |\cP_l|^2 \kappa \sqrt{d_y}}.
  \end{gather*}
  Then, we have with probability at least $1 - \delta_{\P}$ that 
  \[
    \frac{ \W_{N, T}\x_{\bmu} }{ \inprod{\One}{\W_{N, T}\x_{\bmu}} }
    = \q_{\bmu} \pm_2 \eps_*, \quad 
    \forall \bmu \in \cP_l, \, \x_{\bmu} \in \cP_{\rmR, l, \bmu}.
  \]
\end{lemma}
\begin{proof}
  First, by Lemma~\ref{lemma: convergence of gradient descent}, we have 
  \[
    \W_{N, T}\x_{\bmu}
    = \hat{p}_{\bmu}  
      \underbrace{ 
        \left(
          \hat{\E}[ \hat\x_{\bmu} ]\trans
          \left( \hat{p}_{\bmu} \hat\E[ \hat{\x}_{\bmu}\hat{\x}_{\bmu}\trans ] + \lambda_W \Id_{d_x} \right)\inv
          \x_{\bmu}
        \right)
      }_{=: \alpha_{\bmu}}
      \hat{\q}_{\bmu}
      \pm_2 \underbrace{
        \left(
          e^{-T/|\cP_l|} |\cP_l|
          + 6 |\cP_l|^2 \eps_O
        \right) 
      }_{=: \eps_{\Tmp}}
  \]
  By Lemma~\ref{lemma: q coefficient lower bound}, $\alpha_{\bmu} \ge 1/(4(\hat{p}_{\bmu} + \lambda_W))$. In addition, 
  by Lemma~\ref{lemma: concentration of hat q and hat p}, to ensure $\hat{p}_{\bmu} \ge (1 \pm 0.5)p_{\bmu}/2$ with 
  probability at least $1 - \delta_{\P}$, it suffices to choose 
  \[
    N \ge 8 \kappa |\cP_l| \log(4|\cP_l|/\delta_{\P}).
  \]
  Recall from Assumption~\ref{assumption: rhm}\ref{itm: rhm non-degeneracy} that $p_{\bmu} \ge (\kappa|\cP_l|)\inv$.
  Hence, we have 
  \[
    \hat{p}_{\bmu} \alpha_{\bmu}
    \ge \frac{p_{\bmu} / 2}{ 3 p_{\bmu}/2 + \lambda_W }
    \ge \frac{ (\kappa|\cP_l|)\inv }{ 3 (\kappa|\cP_l|)\inv + 2 |\cP_l|\inv }
    \ge \frac{1}{ 3 + 2 \kappa }
    \ge \frac{1}{5 \kappa}.
  \]
  Recall that $\hat{\q}_{\bmu}$ is a probability vector, so $\One\trans\hat{\q}_{\bmu} = 1$ always holds. 
  Hence, for the normalizing factor, we have 
  \[
    \inprod{\One}{\W_{N, T}\x_{\bmu}}
    = \hat{p}_{\bmu} \alpha_{\bmu} \pm \sqrt{d_y} \eps_{\Tmp}
    = \hat{p}_{\bmu} \alpha_{\bmu} 
      \left( 1 \pm 5 \kappa \sqrt{d_y} \eps_{\Tmp} \right).
  \]
  For the (relative) error to be at most $0.2$, it suffices to require 
  \[
    5 \sqrt{d_y} \eps_{\Tmp}
    \le 0.2
    \quad\Leftarrow\quad 
     \eps_{\Tmp}
    \le \left( 100 \kappa \sqrt{d_y} \right)\inv.
  \]
  Then, for the normalized output, we have 
  \begin{align*}
    \frac{ \W_{N, T}\x_{\bmu} }{ \inprod{\One}{\W_{N, T}\x_{\bmu}} }
    &= \frac{ \hat{p}_{\bmu} \alpha_{\bmu} \hat{\q}_{\bmu} \pm_2 \eps_{\Tmp} }{
        \hat{p}_{\bmu} \alpha_{\bmu} 
        \left( 1 \pm 5 \kappa \sqrt{d_y} \eps_{\Tmp} \right)
      } \\
    &= \hat{\q}_{\bmu} \left( 1 \pm 10 \kappa \sqrt{d_y} \eps_{\Tmp} \right) 
      \pm 10 \eps_{\Tmp} \\
    &= \q_{\bmu} 
      \pm_2 \norm{ \hat{\q}_{\bmu} - \q_{\bmu} }
      \pm_2 20 \kappa \sqrt{d_y} \eps_{\Tmp}  .
  \end{align*}
  By Lemma~\ref{lemma: concentration of hat q and hat p}, for $\norm{ \hat{\q}_{\bmu} - \q_{\bmu} } \le \eps_*/2$
  to hold, it suffices to require
  \[
    N \ge \frac{8 \log(4|\cP_l|/\delta_{\P})}{\eps_*^2}.
  \]
  In addition, for the last error term to be bounded by $\eps_*/2$, we can require
  \begin{align*}
    20 \kappa \sqrt{d_y} \eps_{\Tmp}
    \le \frac{\eps_*}{2}
    &\quad\Leftarrow\quad 
    e^{-T/|\cP_l|} |\cP_l| + 6 |\cP_l|^2 \eps_O
    \le \frac{\eps_*}{40 \kappa \sqrt{d_y}} \\
    &\quad\Leftarrow\quad 
    T
    \ge |\cP_l| \log\left( \frac{100 |\cP_l| \kappa \sqrt{d_y}}{\eps_*} \right), \quad 
    \eps_O
    \le \frac{\eps_*}{600 |\cP_l|^2 \kappa \sqrt{d_y}}
  \end{align*}
\end{proof}

\begin{lemma}[Concentration of $\hat{\q}_{\bmu}$ and $\hat{p}_{\bmu}$]
  \label{lemma: concentration of hat q and hat p}
  Let $\delta_{\P} \in (0, 1)$ be our target failure probability. We have with probability at least $1 - \delta_{\P}$
  that 
  \[
    \hat{p}_{\bmu}
    = \left( 1 \pm \sqrt{ \frac{2 \kappa |\cP_l| \log(4|\cP_l|/\delta_{\P})}{ N} } \right) p_{\bmu}, 
    \quad 
    \norm{ \hat\q_{\bmu} - \q_{\bmu} } 
    \le \sqrt{ \frac{2 \log(4|\cP_l|/\delta_{\P})}{N} }, \quad 
    \forall \bmu \in \cP_l.
  \]
\end{lemma}
\begin{proof}
  Fix $\bmu \in \cP_l$. Note that $\hat{p}_{\bmu} \overset{d}{=} N\inv \sum_{i=1}^{N} Z_{\bmu, i}$ where $(Z_{\bmu, i})_i$
  are i.i.d.~$\mrm{Bernoulli}(p_{\bmu})$ variables. Recall from Assumption~\ref{assumption: rhm}\ref{itm: rhm non-degeneracy} 
  that $p_{\bmu} \ge ( \kappa |\cP_l|)\inv$. Therefore, for any $\delta_{\P} \in (0, 1)$, the standard Chernoff yields
  \[
    \hat{p}_\mu 
    = \left( 1 \pm \sqrt{ \frac{2 \kappa |\cP_l| \log(2/\delta_{\P})}{ N} } \right) p_{\bmu}, \quad 
    \text{with probability at least $1 - \delta_{\P}$}.
  \]
  Similarly, since $\norm{ \e_{\hat\zeta} }_2 \le 1$, by the vector version of the Azuma-Hoeffding inequality
  (cf.~Theorem 3.5 of \cite{pinelis_optimum_1994}), we have 
  \[
    \norm{ \hat\q_{\bmu} - \q_{\bmu} } 
    \le \sqrt{ \frac{2 \log(2/\delta_{\P})}{N} }, 
    \quad \text{with probability at least $1 - \delta_{\P}$}.
  \]
  Take union bound over $\bmu \in \cP_l$, and we complete the proof.
\end{proof}

\subsection{Shallow-to-deep chaining}
\label{subsec: appendix: chaining}

In this subsection, we chain the results obtained in the previous subsections to prove Theorem~\ref{thm: opt main}.
First, we prove Lemma~\ref{lemma: training a single layer} (restated below), which determines how small the level-$(l+1)$ accuracy needs to be in order 
for achieving accuracy $\eps\ps{l}_*$ at level $l$ to be possible. This is a direct corollary of Lemma~\ref{lemma: 
error propagation in the RF layer} and Lemma~\ref{lemma: train W}.

\TrainingASingleLayer*
\begin{proof}
  First, by Lemma~\ref{lemma: train W}, we need 
  \begin{gather*}
    N \ge 8 \left( \kappa \vee (\eps\ps{l}_*)^{-2} \right) \log(4|\cP_l|/\delta_{\P}) , \quad 
    T \ge |\cP_l| \log\left( 100 |\cP_l|  \kappa \sqrt{d_y}  / \eps\ps{l}_* \right), \\ 
    \eps_S \le \frac{1}{4 |\cP_l|}, \quad 
    \eps_O \le \frac{\eps\ps{l}_*}{600 |\cP_l|^2 \kappa \sqrt{d_y}}.
  \end{gather*}
  For simplicity, we will choose $\eps_S, \eps_O$ so that the equalities hold. Now, we use Lemma~\ref{lemma: error 
  propagation in the RF layer} to convert these to conditions on $\eps\ps{l+1}_*$. 
  Recall from Assumption~\ref{assumption: rhm}\ref{itm: rhm nontrivial signals} that $\rho\ps{l+1}$ is a lower bound 
  for $\| \q\ps{l+1}_{\bmu} - \q\ps{l+1}_{\bmu'} \| $ when $\bmu, \bmu'$ are not synonyms. Hence, if the level-$(l+1)$
  accuracy $\eps\ps{l+1}_*$ is upper bounded by $\rho\ps{l+1}/2$, we can ensure the level-$l$ intercluster distance 
  is at least $\rho\ps{l+1} - \rho\ps{l+1}/2 = \rho\ps{l+1}/2$. In other words, at level-$l$, Assumption~\ref{assumption: 
  token embeddings}\ref{itm: token intercluster distance} holds with $\tilde{\rho} = \rho\ps{l+1}/2$.
  
  In particular, by Lemma~\ref{lemma: error propagation in the RF layer}, this implies we can choose 
  $\sigma\ps{l}$ as follows to ensure the patch-level orthogonality error is at most $\eps_O\ps{l}$:
  \[
    ( \sigma\ps{l} )^2 
    = \frac{2 ( \rho\ps{l+1}/2 )^2}{\log( 2 / \eps_O\ps{l} )}
    = \frac{( \rho\ps{l+1} )^2}{2 \log\left( 1200 |\cP_l|^2 \kappa \sqrt{d_y} / \eps\ps{l}_*  \right)}
  \]
  This imposes a condition on the intracluster distance $\tilde{\eps}_S$. Again by Lemma~\ref{lemma: error propagation 
  in the RF layer}, it suffices to have 
  \[
    \tilde{\eps}_S 
    \le \frac{\eps_S \sigma\ps{l}}{\sqrt{s}}
    = \frac{ \rho\ps{l+1} }{\sqrt{ 32 |\cP_l|^2 s \log\left( 1200 |\cP_l|^2 \kappa \sqrt{d_y} / \eps\ps{l}_*  \right)}}. 
  \]
  Note that the intracluster distance is at most $2 \eps\ps{l+1}_*$. Meanwhile, recall from Assumption~\ref{assumption: rhm}
  that our RHM is $(V, m)$-uniform, so $|\cP_l| = V m$, and $d_y = |\cV_0| = V$. Thus, it is sufficient to require
  \[
    \eps\ps{l+1}_*
    \le \frac{ \rho\ps{l+1} }{\sqrt{ 200 V^2 m^2 s \log\left( 1200 V^{2.5} m^2 \kappa / \eps\ps{l}_*  \right)}}
    \;\Leftarrow\;
    \eps\ps{l+1}_*
    \le \frac{ \rho\ps{l+1} }{\sqrt{ 2000 V^2 m^2 s \log\left( V m \kappa / \eps\ps{l}_* \right)}}.
  \]
  Finally, the numbers of needed samples and gradient steps follow directly from Lemma~\ref{lemma: train W} and 
  the model width requirement comes from Lemma~\ref{lemma: error propagation in the RF layer}.
\end{proof}

Now, we prove our main optimization result (Theorem~\ref{thm: opt main}) using the above lemma. 
\OptMain*
\begin{proof}
  First, by the non-ambiguity, we can recover the ground-truth label as long as $\eps\ps{1}_* \le c$ for some small 
  constant $c > 0$. Recall from Assumption~\ref{assumption: rhm} that $\rho\ps{l} = K_\rho m^{-l/2}$. 
  By Lemma~\ref{lemma: training a single layer}, it suffices to choose the target accuracies so that 
  \[
    \eps\ps{l+1}_*
    \le \frac{ \rho\ps{l+1} }{\sqrt{ 2000 V^2 m^2 s \log( V m \kappa / \eps\ps{l}_* )}}
    = \frac{ K_\rho m^{-(l+1)/2} }{\sqrt{ 2000 V^2 m^2 s \log( V m \kappa / \eps\ps{l}_* )}}. 
  \]
  To this end, we can choose 
  \[
    \eps\ps{l}_*
    = \frac{ K_\rho m^{-l/2} }{ \sqrt{ C V^2 m^2 s  L\log m }}, 
  \]
  for some sufficiently large universal constant $C \ge 1$. By Lemma~\ref{lemma: training a single layer}, 
  this implies the number of samples needed at level $l$ satisfies
  \[
    N\ps{l} \le 20 C \kappa V^2 m^2 s L\log^2(V m \kappa /\delta_{\P}) K_\rho^{-2} m^l .
  \]
  Note that this is a geometric sequence. Thus, $\sum_{l=1}^{L} N\ps{l} \lesssim 
  \kappa V^2 m^2 s L\log^2(V m \kappa /\delta_{\P}) K_\rho^{-2} m^L$. The bounds on the model width and number of 
  gradient steps follow directly from the corresponding results in Lemma~\ref{lemma: training a single layer}.
\end{proof}

\newpage
\section{Signal lower bounds}
\label{sec: signal lower bounds}

In this section, we prove the following proposition, which verifies Assumption~\ref{assumption: rhm}
in the case where $L \sim \log d/\log\log d$ and $m, V, s \sim \log d$, and the production rules are randomly sampled, 
where $d$ is the length of the input sequence,\footnote{Note that $(\log d)^{\log d / \log\log d} = d$.} and the productions are sampled without replacement uniformly at random 
while maintaining the marginal distributions of the first tokens being uniform. 
Namely, for each $l \in [L]$ and $\mu \in \cV_l$, we first sample $\bmu_{1, \ge 2}, \dots, \bmu_{m \ge 2}$ from $\cV_l^{s-1}$ 
uniformly at random without replacement, and put $\{ (\mu, \bmu_{i, \ge 2}) \}_{i \in [m]} \subset \cV_l^s$ into the 
collection of candidate patches. Then, we associate the candidate patches randomly with $\nu \in \cV_{l-1}$ so that 
each $\nu$ is associated with exactly $m$ patches. For the distribution of the labels, we assume $\cD_0 = \Unif(\cV_0)$.

\begin{restatable}{rProposition}{rhoLowerBound}
  \label{prop: signal lower bound}
  Suppose that $m, V = \Theta(\log d)$ and $L = O(\log d/ \log\log d)$ and $\min\{m, V\} \gg \log^3(1/\delta_{\P})$.
  Suppose that the RHM productions are randomly sampled while maintaining the marginal distribution of the first token 
  at each level is uniform over the vocabulary.\footnote{Maintaining the uniformity will only decrease
  the signal size, since we will lose the variance coming from the denominator. We choose to maintain the uniformity as 
  it leads to a cleaner analysis.} 
  Then, with probability at least $1 - \delta_{\P}$, we have 
  $\| \q\ps{l}_{\bmu} - \q\ps{l}_{\bmu'} \| \ge (20 m)^{-(l-1)/2} \ge d^{-o(1)} m^{-(l-1)/2}$ for all $l \in [L]$ 
  and non-synonyms $\bmu, \bmu' \in \cP_l$. 
\end{restatable}

Since Assumption~\ref{assumption: rhm} concerns only the first patch at each level and the first patch is generated 
by the first token of the previous level, it suffices to consider the conditional distribution of the label given the 
first token. 

First, we introduce several notations. 
For a given RHM instance, let $\bP_l: \cV_{l-1} \times \cV_l \to \R_{\ge 0}$ denote the transition function 
at level $l-1$. That is, for any $\mu \in \cV_l$ and $\nu \in \cV_{l-1}$, $P_l(\nu, \mu)$ is the probability that 
the first token at level $l$ is $\mu$, given that the first token at level $l-1$ is $\nu$. We will also abuse the 
notations and use $\bP_l$ to denote the corresponding transition matrix. 

For any label $\zeta \in \cV_0$ and level-$l$ token $\mu$, we have 
\[
  \P\big[ \mu\ps{l}_1 = \mu \mid \text{label is } \zeta \big]
  = \e_\zeta\trans \bP_1 \cdots \bP_l \e_\mu.
\]
By the Bayes rule, this implies  
\[
  \P\big[ \text{label is } \zeta \mid \mu\ps{l}_1 = \mu \big]
  = \frac{
      \P\big[ \mu\ps{l}_1 = \mu \mid \text{label is } \zeta ]
      \P\big[ \text{label is } \zeta ]
    }{
      \sum_{\zeta' \in \cV_0} \P\big[ \mu\ps{l}_1 = \mu \mid \text{label is } \zeta_0 ]
      \P\big[ \text{label is } \zeta_0 ]
    } \\
  = \frac{ \e_\zeta\trans \bP_1 \cdots \bP_l \e_\mu }{ \One\trans \bP_1 \cdots \bP_l \e_\mu }.
\]
By definition, for any $k \in [l]$, $\bP_k$ is row stochastic.
In addition, since $\sum_\nu P_k(\nu, \mu)$ is equal to ratio of possible level-$k$ patches whose first token is $\mu$, 
by our sampling procedure, this is $1/V$. Therefore, $\bP_k$ is also column stochastic. In particular, this implies that 
$\One\trans\bP_k = 1$ and, therefore, $\One\trans \bP_1 \cdots \bP_l \e_\mu = 1$. Thus, we have 
\[
  \y\ps{l}_\mu
  := \left( \P\big[ \text{label is } \zeta \mid \mu\ps{l}_1 = \mu \big] \right)_{\zeta \in \cV_0}
  = \bP_1 \cdots \bP_l \e_\mu. 
\]
Note that the distribution of $\bP_k$ can be described using the balls-into-bins model as follows. 
The statement of it is essentially its proof. 

\begin{lemma}
  \label{lemma: P to balls into bins}
  Suppose that we have $V$ colors, labeled with $\mu \in [V] \cong \cV_l$. 
  For each color $\mu$, there are $m$ color-$\mu$ balls. 
  In addition, we have $V$ different bins, labeled with $\nu \in [V] \cong \cV_{l-1}$, and each of them has $m$ slots. 
  We randomly distribute the balls into the bins. 
  Let $N_{\nu, \mu}$ denote the number of color-$\mu$ balls in bin $\nu$. 
  Define $\bP \in \R^{V \times V}$ by $P_{\nu, \mu} = N_{\nu, \mu} / m$. 
  Then, we have $\bP_1, \dots, \bP_L \overset{\mrm{i.i.d.}}{\sim} \mrm{Law}(\bP)$.
\end{lemma}

Our goal is to lower bound $\norm{ \y\ps{l}_\mu - \y\ps{l}_{\mu'} }$ for $\mu \ne \mu'$.
Note that for any $\mu \ne \mu' \in \cV_l$, $\One\trans(\e_{\mu} - \e_{\mu'}) = 0$ and for any $\z \in \R^V$ with 
$\One\trans\z = 0$, we have $\One\trans(\bP_k\z) = \One\trans\z = 0$. As a result, it suffices to lower bound 
$\norm{ \bP\z }$ using $\|\z\|$ for $\z$ with $\One\trans\z = 0$, and then inductively lower bound 
$\norm{ \bP_1 ( \bP_2 \cdots \bP_l \z) }$.

\subsection{One-step concentration}

Let $\z \in \R^V$ be such that $\One\trans\z = 0$ and $\bP \in \R^{V \times V}$ be the random transition matrix described
in Lemma~\ref{lemma: P to balls into bins}. 
In this subsection, we will always identify $\cV_{l-1}$ with $[V]$ and this also induces an ordering of the bins. 
Define $\btau := \bP \z$. Note that, by definition, 
\[
  \tau_\nu 
  = [\bP\z]_\nu
  = \frac{1}{m} \sum_{\mu \in \cV_l} N_{\nu, \mu} z_\mu, \quad 
  \forall \nu \in \cV_{l-1}. 
\]

First, we show that $\tau_\nu$ can be equivalently viewed as a sample mean obtained by sampling without replacement. 
\begin{lemma}
  \label{lemma: sum N x = sum X}
  Let the population be the multiset 
  \[
    \cZ := \bigcup_{\mu \in \cV_l} \big\{ \underbrace{ z_\mu, \dots, z_\mu }_{m \text{ copies}} \big\}, \quad 
    |\cZ| = m V.
  \]
  Let $\{ Z_{\nu, s} \}_{\nu \in \cV_{l-1}, s \in [m]}$ be obtained by sampling without replacement from $\cZ$. We have 
  \[
    \btau 
    = \left( \frac{1}{m} \sum_{\mu \in \cV_l} N_{\nu, \mu} z_\mu \right)_{\nu \in \cV_{l-1}}
    \overset{d}{=} \left( \frac{1}{m} \sum_{s=1}^{m} Z_{\nu, s} \right)_{\nu \in \cV_{l-1}}.
  \]
\end{lemma}
\begin{proof}
  For $\mu \in \cV_l$, we assign value $z_\mu$ to each color-$\mu$ ball. Meanwhile, 
  for each bin $\nu \in \cV_{l-1}$ and slot $s \in [m]$, let $Z_{\nu, s}$ be the value of the ball in slot $(v, s)$. 
   Note that 
  $\sum_{s=1}^{m} Z_{\nu, s}$ is the sum of the values of the balls in bin $\nu$, which, by definition, is equal to 
  $\sum_{\mu \in \cV_l} N_{\nu, \mu} z_\mu$. Moreover, $\{ Z_{\nu, \mu} \}_{\nu, \mu}$ can be obtained by sampling 
  without replacement from the universe $\cZ$. 
\end{proof}

As the first application of the above formula, we now derive bounds on the conditional expectation
and variance of $\tau_\nu^2$. We will need the following Hoeffding's comparison lemma. In words, it says that 
for convex test functions, without-replacement samples are more concentrated than with-replacement samples. 
\begin{lemma}[Theorem 4 of \cite{hoeffding_probability_1963}]
  \label{lemma: hoeffding comparison}
  Let $\cZ \subset \R$ be a finite multiset. Fix $n \le |\cZ|$. 
  Let $X_1^{\wr}, \dots, X_n^{\wr}$ and $X_1^{\wor}, \dots, X_n^{\wor}$ be with- and without-replacement 
  samples from $\cZ$. Then, for any convex $g: \R \to \R$, we have 
  \[
    \E g\left( \sum_{i=1}^{n} X_i^{\wor} \right)
    \le \E g\left( \sum_{i=1}^{n} X_i^{\wr} \right)
  \]
\end{lemma}

One can apply the above lemma to moment generating functions to obtain following Chernoff bound. 
\begin{corollary}
  \label{cor: sample with replacement, more concentrated}
  Let $\cZ \subset \R$ be a finite multiset with $|\cZ|\inv \sum_{x \in \cZ} x =: \mu_{\cZ}$. Fix $n \le |\cZ|$. 
  Let $X_1^{\wr}, \dots, X_n^{\wr}$ and $X_1^{\wor}, \dots, X_n^{\wor}$ be with- and without-replacement 
  samples from $\cZ$. Then, for any $t \ge 0$, we have 
  \begin{align*}
    \P\left( \sum_{i=1}^{n} X_i^{\wor} - n \mu_{\cZ} \ge t  \right)
    &\le \inf_{\lambda \ge 0} e^{-\lambda t} \left( \E e^{\lambda \left( X_1^{\wr} - \mu_{\cZ} \right)} \right)^n, \\
    \P\left( \sum_{i=1}^{n} X_i^{\wor} - n \mu_{\cZ} \le - t  \right)
    &\le \inf_{\lambda \ge 0} e^{-\lambda t} \left( \E e^{-\lambda \left( X_1^{\wr} - \mu_{\cZ} \right)} \right)^n.
  \end{align*}
\end{corollary}

\begin{lemma}[Number of used balls]
  \label{lemma: number of used balls}
  Identify $\cV_{l-1}$ with $[V]$ and let $N_{\le \nu, \mu}$ denote the number of color-$\mu$ balls used in the 
  first $\nu$ bins. Suppose that $\gamma \in (0, 1/2)$ and $\eps \le 0.1$. Then, for each fixed $\mu \in \cV_l$, 
  we have 
  \[
    \P\left( |N_{\le \gamma V, \mu} - \gamma m| \ge \eps \gamma m  \right)
    \le 2 e^{-\eps^2 \gamma m / 3}. 
  \]
\end{lemma}
\begin{remark}
  We will eventually choose $\gamma = \Theta(1)$ and take union bound over $\poly(V, m)$ such events. 
  In the regime $V, m = \Theta(\log d)$, the overall failure probability can be bounded 
  $e^{ \Theta(\log\log d) - \eps^2 \Theta( \log d ) }$. In particular, for any $\nu \in [V/3, 2V/3]$ and 
  $\mu \in \cV_l$, by choosing $\gamma = \nu/V$, we have 
  \begin{equation}
    \label{eq: N <= nu mu approx nu m / V}
    N_{\le \nu, \mu} 
    = \frac{\nu m}{V} \pm \sqrt{ 2m\log(2/\delta_{\P}) }, 
    \quad\text{with probability at least $1 - \delta_{\P}$.}
  \end{equation}
\end{remark}
\begin{proof}
  Let $\gamma \in (0, 1/2)$ be a small constant to be determined later. Define $Z_{\nu, s} = 
  \indi\{ Z_{\nu, s} = z_\mu \}$ be the indicator of the event that the ball in slot $(\nu, s)$ is of color $\mu$. 
  Note that $N_{\le \gamma V, \mu} = \sum_{\nu=1}^{\gamma V} \sum_{s=1}^{m} Z_{\nu, s}$ and $\{ Z_{\nu, s} \}_{\mu, s}$
  are without-replacement samples from the multiset consisting of $m$ ones and $(V-1)m$ zeros. Hence, by 
  Corollary~\ref{cor: sample with replacement, more concentrated}, we have 
  \begin{align*}
    \P\left( N_{\le \gamma V, \mu} - \gamma m \ge \eps \gamma m  \right)
    &\le \inf_{\lambda \ge 0} e^{-\lambda \eps \gamma m} \left( \E e^{\lambda \left( \mrm{Bernoulli}(1/V) - 1/V \right)} \right)^{\gamma m V} \\
    &= \inf_{\lambda \ge 0} e^{-\lambda \eps \gamma m} \left( 
        \left( 1 + (e^{\lambda} - 1) / V \right)
        e^{- \lambda/V} 
      \right)^{\gamma m V} \\
    &\le \inf_{\lambda \ge 0} 
        \exp\left( 
          - (1 + \eps) \lambda \gamma m
          + (e^{\lambda} - 1) \gamma m 
        \right).
  \end{align*}
  Suppose that $\eps \le 0.1$. Pick $\lambda = \log(1+\eps) > 0$, and we obtain 
  \[
    \P\left( N_{\le \gamma V, \mu} - \gamma m \ge \eps \gamma m  \right)
    \le \exp\left( 
          - (1 + \eps) \log(1+\eps) \gamma m
          + \eps \gamma m 
        \right)
    \le \exp\left( - \eps^2 \gamma m / 3 \right),
  \]
  where the last inequality comes from Taylor expanding $\eps \mapsto (1 + \eps)\log(1 + \eps)$ around $0$. 
  By the same argument and $\lambda$, we get  
  \begin{align*}
    \P\left( N_{\le \gamma V, \mu} - \gamma m \le - \eps \gamma m  \right)
    &\le \inf_{\lambda \ge 0} e^{-\lambda \eps \gamma m} \left( \E e^{-\lambda \left( \mrm{Bernoulli}(1/V) - 1/V \right)} \right)^{\gamma m V} \\
    &= \inf_{\lambda \ge 0} 
        \exp\left( 
          (e^{-\lambda} - 1) \gamma m 
          + (1 - \eps) \lambda \gamma m 
        \right) \\
    &\le \exp\left( 
          \left( \frac{1}{1 + \eps}  - 1 \right) \gamma m 
          + (1 - \eps) \log(1 + \eps) \gamma m 
        \right) 
    \le \exp\left( - \eps^2 \gamma m \right). 
  \end{align*}
\end{proof}

\begin{lemma}[Conditional expectation and variance of $X$]
  \label{lemma: X conditional expectation and variance}
  Let $(\cF_\nu)_\nu$ be the $\sigma$-algebra generated by the first $\nu$ bins. 
  Let $\bar{z}_\nu := \E[ Z_{\nu, 1} \mid \cF_{\nu-1} ]$ and $\sigma_{x, \nu}^2 := \Var[ Z_{\nu, 1} \mid \cF_{\nu-1} ] $
  denote the conditional mean and variance of the values of a random ball in the remaining balls after $\nu - 1$ bins. 
  Then, we have 
  \begin{align*}
    \Cov(Z_{\nu, 1}, Z_{\nu, 2} \mid \cF_{\nu-1})
    = - \frac{\sigma_{x, \nu}^2}{m(V-\nu+1)-1}, \\
    \E[ \tau_\nu^2 \mid \cF_{\nu-1} ]
    = \frac{\sigma_{x, \nu}^2 }{m} \left( 1 - \frac{m-1}{m(V-\nu+1)-1}  \right) + \bar{z}_\nu^2.
  \end{align*}
\end{lemma}
\begin{proof}
  Let $\cX_\nu$ denote the remaining values of $\cZ$ after filling up $\nu - 1$ bins. 
  Note that $|\cX_\nu| = m (V - \nu + 1)$ and we have 
  \begin{align*}
    \Cov(Z_{\nu, 1}, Z_{\nu, 2} \mid \cF_{\nu-1})
    &= \E[ Z_{\nu, 1} Z_{\nu, 2} \mid \cF_{\nu-1} ] - \E[ Z_{\nu, 1} \mid \cF_{\nu-1} ]^2 \\
    &= \frac{1}{|\cX_\nu|( |\cX_\nu| - 1 )} \sum_{\text{$x, x'$: different balls in $\cX_\nu$}} x x'
      - \bar{z}_\nu^2 \\
    &= \frac{1}{|\cX_\nu|( |\cX_\nu| - 1 )} \sum_{x \in \cX_\nu} x \left( \sum_{x' \in \cX_\nu} x' - x \right) 
      - \bar{z}_\nu^2 \\
    &= \frac{ \left( \sum_{x \in \cX_\nu} x \right)^2 - \sum_{x \in \cX_\nu} x^2 }{|\cX_\nu|( |\cX_\nu| - 1 )} 
      - \frac{\left( \sum_{x \in \cX_\nu} x \right)^2}{|\cX_\nu|^2}   \\
    &= \frac{\left( \sum_{x \in \cX_\nu} x \right)^2 }{|\cX_\nu|^2( |\cX_\nu| - 1 )} 
      - \frac{\sum_{x \in \cX_\nu} x^2}{|\cX_\nu|( |\cX_\nu| - 1 )}  \\
    &= \frac{ \bar{z}_\nu^2 }{ |\cX_\nu| - 1 } 
      - \frac{ \E[ Z_{\nu, 1}^2 \mid \cF_{\nu-1} ] }{ |\cX_\nu| - 1 } 
    = - \frac{\sigma_{x, \nu}^2}{|\cX_\nu|-1}, 
  \end{align*}
  Then, for $\E[ \tau_\nu^2 \mid \cF_{\nu-1} ]$, by symmetry, we have 
  \begin{align*}
    \E[ \tau_\nu^2 \mid \cF_{\nu-1} ]
    &= \frac{1}{m^2} \sum_{s \in [m]} \E[ Z_{\nu, s}^2 \mid \cF_{\nu-1} ]
      + \frac{1}{m^2} \sum_{s \ne s' \in [m]} \E[ Z_{\nu, s} Z_{\nu, s'} \mid \cF_{\nu-1} ] \\
    &= \frac{1}{m} \E[ Z_{\nu, 1}^2 \mid \cF_{\nu-1} ]
      + \frac{m-1}{m} \E[ Z_{\nu, 1} Z_{\nu, 2} \mid \cF_{\nu-1} ] \\
    &= \frac{\sigma_{x, \nu}^2 + \bar{z}_\nu^2}{m} 
      + \frac{m-1}{m} \left( \Cov(Z_{\nu, 1}, Z_{\nu, 2} \mid \cF_{\nu-1}) + \bar{z}_\nu^2 \right).  
  \end{align*}
  Then, using the previous formula on the covariance, we obtain
  \begin{align*}
    \E[ \tau_\nu^2 \mid \cF_{\nu-1} ]
    &= \frac{\sigma_{x, \nu}^2 + \bar{z}_\nu^2}{m} 
      + \frac{m-1}{m} \left( - \frac{\sigma_{x, \nu}^2}{m(V-\nu+1)-1} + \bar{z}_\nu^2 \right) \\
    &= \frac{\sigma_{x, \nu}^2 }{m} \left( 1 - \frac{m-1}{m(V-\nu+1)-1}  \right)  
      + \bar{z}_\nu^2. 
  \end{align*}
\end{proof}

Note that both $\sigma_{x, \nu}^2$ and $\bar{z}_\nu$ are still random variables. 
We now use Lemma~\ref{lemma: number of used balls} to derive high probability bounds for them and the mean and variance 
of $\tau_\nu^2$.

\begin{lemma}[Concentration of conditional mean and variance of $\tau_\nu^2$]
  \label{lemma: ynu2 conditional mean and var concentration}
  Let $\cF_\nu$ be the $\sigma$-algebra generated by the first $\nu$ bins. 
  Fix $\nu \in [V/3, 2V/3]$, for any $\delta_{\P} \in (0, 1)$, we have with probability at least $1 - \delta_{\P}$ 
  that 
  \begin{align*}
    \E[ \tau_\nu^2 \mid \cF_{\nu-1} ]
    &\ge \left( 1 \pm \frac{60\log(2 V / \delta_{\P})}{m } \right) 
      \frac{\norm{\z}^2}{5 m V},  \\
    \E[ \tau_\nu^2 \mid \cF_{\nu-1} ]
    &\le 60 \log(2 V / \delta_{\P}) \frac{\norm{\z}_2^2}{m V}, \\
    \E[ \tau_\nu^4 \mid \cF_{\nu-1} ]
    &\lesssim \log^2(2 V / \delta_{\P})  \frac{\norm{\z}^4}{m^2 V^2}.
  \end{align*}
\end{lemma}
\begin{proof}
  For each $\nu \in \cV_{l-1}, \mu \in \cV_l$, let $H_{\nu, \mu} = m - N_{\le \nu-1, \mu}$ denote the number of 
  unused color-$\mu$ balls before we start to fill bin $\nu$. 
  By \eqref{eq: N <= nu mu approx nu m / V} and union bound, we know 
  with probability at least $1 - \delta_{\P}$, 
  \[
    H_{\nu, \mu} 
    = \left( 1 - \frac{\nu-1}{V}\right) m  \pm \sqrt{ 2 m \log(2 V/\delta_{\P}) }
    =: \left( 1 - \frac{\nu-1}{V}\right) m  \pm \delta_H, 
    \quad \forall \mu \in \cV_l,
  \]
  where $\delta_H = \tilde{\Theta}(\sqrt{m})$.
  Consider $\nu \in [V/3, 2V/3]$. Using the above bound and the condition $\sum_\mu z_\mu = 0$, we can derive 
  \begin{align*}
    \abs{ \bar{z}_\nu  }
    = \abs{ \frac{1}{m(V-\nu+1)} \sum_{\mu \in \cV_l} H_{\nu, \mu} z_\mu  }
    &= \abs{ 
        \frac{1}{m(V-\nu+1)} 
        \sum_{\mu \in \cV_l} 
        \left(
          \left( 1 - \frac{\nu-1}{V}\right) m  \pm \delta_H
        \right) z_\mu  
      } \\
    &\le \frac{1}{m(V-\nu+1)} \sum_{\mu \in \cV_l} \delta_H |z_\mu| 
    \le \frac{3 \delta_H}{m V} \norm{\z}_1
    \le \frac{3 \delta_H}{m \sqrt{V}} \norm{\z}_2.
  \end{align*}
  Similarly, for the second moment, we have 
  \begin{align*}
    \E[ Z_{\nu, 1}^2 \mid \cF_{\nu-1} ]
    &= \frac{1}{m(V-\nu+1)} \sum_{\mu \in \cV_l} H_{\nu, \mu} z_\mu^2 \\
    &= \frac{1}{m(V-\nu+1)} 
      \sum_{\mu \in \cV_l} \left( \left( 1 - \frac{\nu-1}{V}\right) m \pm \delta_H \right) z_\mu^2 \\
    &= \frac{1 - (\nu-1)/V}{V-\nu+1} \norm{\z}^2
      \pm  \frac{3 \delta_H}{mV}  \norm{\z}^2.
  \end{align*}
  Therefore, we have 
  \begin{align*}
    \sigma_{x, \nu}^2 
    = \E[ Z_{\nu, 1}^2 \mid \cF_{\nu-1} ] - \bar{z}_\nu^2 
    &= \frac{1 - (\nu-1)/V}{V-\nu+1} \norm{\z}^2
      \pm \frac{3 \delta_H}{mV}  \norm{\z}^2
      \pm \frac{9 \delta_H^2}{m^2 V} \norm{\z}_2^2 \\
    &= \frac{1 - (\nu-1)/V}{V-\nu+1} \norm{\z}^2
      \pm \frac{10 \delta_H^2}{m^2 V} \norm{\z}_2^2 \\
    &= \left(
      1 \pm \frac{30 \delta_H^2}{m^2}
    \right)
    \frac{1 - (\nu-1)/V}{V-\nu+1} \norm{\z}^2.
  \end{align*}
  Now, we are ready to estimate $\E[ \tau_\nu^2 \mid \cF_{\nu-1} ]$ and $\E[ \tau_\nu^4 \mid \cF_{\nu-1} ]$.
  For the second moment, recall from Lemma~\ref{lemma: X conditional expectation and variance} that 
  \[
    \E[ \tau_\nu^2 \mid \cF_{\nu-1} ]
    = \frac{\sigma_{x, \nu}^2 }{m} \left( 1 - \frac{m-1}{m(V-\nu+1)-1}  \right) + \bar{z}_\nu^2.
  \]
  Therefore, for the lower bound, we have 
  \begin{align*}
    \E[ \tau_\nu^2 \mid \cF_{\nu-1} ]
    &\ge \frac{\sigma_{x, \nu}^2 }{m} \left( 1 - \frac{m-1}{m(V-\nu+1)-1}  \right) \\
    &= \left( 1 \pm \frac{30 \delta_H^2}{m^2} \right)
      \frac{1}{m} \left( 1 - \frac{m-1}{m(V-\nu+1)-1}  \right) 
      \frac{1 - (\nu-1)/V}{V-\nu+1} \norm{\z}^2 \\
    &= \left( 1 \pm \frac{30 \delta_H^2}{m^2} \right)
      \left( 1 - \frac{m-1}{m(V-\nu+1)-1}  \right) 
      \left( 1 - \frac{\nu-1}{V} \right)
      \frac{\norm{\z}^2}{m (V-\nu+1) } \\
    &\ge \left( 1 \pm \frac{30 \delta_H^2}{m^2} \right)
      \frac{\norm{\z}^2}{5 m V}, 
  \end{align*}
  where we have used $\nu \in [V/3, 2V/3]$ and $V \gg 1$ in the last line.
  For the upper bound, by our previous bound on $\bar{z}_\nu$, we have 
  \begin{align*}
    \E[ \tau_\nu^2 \mid \cF_{\nu-1} ]
    &\le 
      \left( 1 \pm \frac{30 \delta_H}{m} \right)
      \left( 1 - \frac{m-1}{m(V-\nu+1)-1}  \right) 
      \left( 1 - \frac{\nu-1}{V} \right)
      \frac{\norm{\z}^2}{m (V-\nu+1) }
      + \left( \frac{3 \delta_H}{m \sqrt{V}} \norm{\z}_2 \right)^2 \\
    &\le 
      \left( 1 \pm \frac{30 \delta_H}{m} \right)
      \frac{3 \norm{\z}^2}{m V }
      + \frac{9 \delta_H^2}{m} \frac{3 \norm{\z}_2^2}{m V} \\
    &\le \frac{30 \delta_H^2}{m} \frac{\norm{\z}_2^2}{m V}. 
  \end{align*}
  Now, consider the second moment of $\tau_\nu^2$. By symmetry and the negative association of without-replacement samples, we have 
  \begin{align*}
    \E[ \tau_\nu^4 \mid \cF_{\nu-1} ]
    &= \frac{1}{m^4} \sum_{s_1, s_2, s_3, s_4=1}^m \E[ Z_{\nu, s_1} Z_{\nu, s_2} Z_{\nu, s_3} Z_{\nu, s_4} \mid \cF_{\nu-1} ] \\
    &\lesssim \frac{1}{m^3} \E[ Z_{\nu, 1}^4 \mid \cF_{\nu-1} ]
      + \frac{1}{m^2} \E[ Z_{\nu, 1}^3 Z_{\nu, 2} \mid \cF_{\nu-1} ]
      + \frac{1}{m^2}\E[ Z_{\nu, 1}^2 Z_{\nu, 2}^2 \mid \cF_{\nu-1} ] \\
      &\qquad
      + \frac{1}{m}\E[ Z_{\nu, 1}^2 Z_{\nu, 2} Z_{\nu, 3} \mid \cF_{\nu-1} ]
      + \E[ Z_{\nu, 1} Z_{\nu, 2} Z_{\nu, 3} Z_{\nu, 4} \mid \cF_{\nu-1} ] \\
    &\lesssim 
      \frac{1}{m^3} \E[ Z_{\nu, 1}^4 \mid \cF_{\nu-1} ]
      + \frac{1}{m^2} \E[ Z_{\nu, 1}^3 \mid \cF_{\nu-1} ] \E[ Z_{\nu, 1} \mid \cF_{\nu-1} ]
      + \frac{1}{m^2}\E[ Z_{\nu, 1}^2 \mid \cF_{\nu-1} ]^2 \\
      &\qquad
      + \frac{1}{m} \E[ Z_{\nu, 1}^2 \mid \cF_{\nu-1} ]
        \E[ Z_{\nu, 1} \mid \cF_{\nu-1} ]^2
      + \E[ Z_{\nu, 1} \mid \cF_{\nu-1} ]^4. 
  \end{align*}
  Recall from our previous calculation that 
  \[
    \abs{ \E[ Z_{\nu, 1} \mid \cF_{\nu-1} ] }
    = |\bar{z}_\nu| 
    \lesssim \frac{\delta_H}{m \sqrt{V}} \norm{\z}.
  \]
  In addition, for any $p \ge 2$, we have 
  \[
    \abs{ \E[ Z_{\nu, 1}^p \mid \cF_{\nu-1}  ]  }
    \le \frac{3}{V} \sum_{\mu}  |z_\mu|^p
    \le \frac{3}{V} \norm{\z}_\infty^{p-2} \norm{\z}^2 
    \lesssim \frac{\norm{\z}^p}{V} .  
  \]
  Hence, 
  \begin{align*}
    \E[ \tau_\nu^4 \mid \cF_{\nu-1} ]
    &\lesssim 
      \frac{\norm{\z}^4}{m^3 V} 
      + \frac{1}{m^2 V} \norm{\z}^3 \frac{\delta_H}{m \sqrt{V}} \norm{\z}
      + \frac{1}{m^2 V^2} \norm{\z}^4 
      + \frac{1}{m V} \norm{\z}^2 
        \frac{\delta_H^2}{m^2 V} \norm{\z}^2
      + \frac{\delta_H^4}{m^4 V^2} \norm{\z}^4 \\
    &\lesssim 
      \frac{\norm{\z}^4}{m^2 V}
      \left( \frac{1}{m} \vee \frac{\delta_H^4}{m^2 V} \right).
  \end{align*}
  Finally, plug in $\delta_H := \sqrt{2 m \log(2 V / \delta_{\P})}$, and we get 
  \begin{align*}
    \E[ \tau_\nu^2 \mid \cF_{\nu-1} ]
    &\ge \left( 1 \pm \frac{60\log(2 V / \delta_{\P})}{m } \right) 
      \frac{\norm{\z}^2}{5 m V},  \\
    \E[ \tau_\nu^2 \mid \cF_{\nu-1} ]
    &\le 60 \log(2 V / \delta_{\P}) \frac{\norm{\z}_2^2}{m V}, \\
    \E[ \tau_\nu^4 \mid \cF_{\nu-1} ]
    &\lesssim 
      \frac{\norm{\z}^4}{m^2 V}
      \left( \frac{1}{m} \vee \frac{\log^2(2 V / \delta_{\P}) }{V} \right)
    \lesssim \log^2(2 V / \delta_{\P})  \frac{\norm{\z}^4}{m^2 V^2}.
  \end{align*}
\end{proof}

We will now establish a Bernstein-style lower bound on $\norm{\btau}^2$. To this end, we will need the following simple 
lemma, which appears in most proofs of the standard Bernstein inequality. 

\begin{lemma}
  \label{lemma: bennett, intermediate step}
  Let $W$ be a random variable with $\E W \le 0$ and $W \le b$ a.s. Then, for any $\lambda \in [0, 1/b]$, we have 
  $\E e^{\lambda W} \le \exp( \lambda^2 \E W^2 )$.
\end{lemma}
\begin{proof}
  Define $\phi(z) = (e^z - 1 - z) / z^2$, so that we can write $e^z = 1 + z + z^2 \phi(z)$. 
  It is known that $\phi$ is increasing on $\R$ and $\phi(1) \le 1$ (see any standard proof of the Bennett inequality). 
  Therefore, for any $\lambda \in [0, 1/b]$, we have 
  \[
    \E e^{\lambda W}
    = 1 + \lambda \E W + \lambda^2 \E[ W^2 \phi(\lambda W) ]
    \le 1 + \lambda^2 \E W^2 
    \le \exp( \lambda^2 \E W^2 ).
  \]
\end{proof}

\begin{lemma}[One-step concentration]
  \label{lemma: rhm stats, one step concentration}
  Given $\delta_{\P} \in (0, 1)$, suppose that $\min\{m, V\} \gg \log^3(1/\delta_{\P})$. Then, 
  \[
    \norm{\btau}^2 
    \ge \frac{\norm{\z}^2}{20 m}, \quad 
    \text{with probability at least $1 - O(\delta_{\P})$.}
  \]
\end{lemma}
\begin{proof}
  Clear that to lower bound $\norm{\btau}^2$, it suffices to lower bound $\sum_{\nu=V/3}^{2V/3} \tau_\nu^2$. By 
  Lemma~\ref{lemma: ynu2 conditional mean and var concentration} and union bound, we have with probability at least 
  $1 - \delta_{\P}$ that 
  \[
    \E[ \tau_\nu^2 \mid \cF_{\nu-1} ]
    \le C \log(V / \delta_{\P}) \frac{\norm{\z}_2^2}{m V}
    =: \xi_*, \quad 
    \E[ \tau_\nu^4 \mid \cF_{\nu-1} ]
    \le C \log^2( V / \delta_{\P})  \frac{\norm{\z}^4}{m^2 V^2}
    =: \sigma_*^2,
  \]
  for all $\nu \in [V/3, 2V/3]$, where $C$ is some universal constant. 
  Note that both $\xi_*$ and $\sigma_*^2$ are deterministic quantities.
  Define 
  \[
    \iota_\nu 
    := \indi\braces{ 
        \E\left[ \tau_\nu^2 \mid \cF_{\nu-1} \right] \le \xi_*, 
        \E\left[ \tau_\nu^4 \mid \cF_{\nu-1} \right] \le \sigma_*^2 
      }, \quad 
    \tilde{\tau}_\nu^2 
    :=  \tau_\nu^2 \iota_\nu. 
  \]
  Note that $\iota_\nu$ is $\cF_{\nu-1}$-measurable and, therefore, the conditional expectation and second moment 
  of $\tilde{\tau}_\nu^2$ are bounded by $\xi_*$ and $\sigma_*^2$, respectively. Moreover, clear that 
  $\tau_\nu^2 \ge \tilde{\tau}_\nu^2$ a.s. 
  Then, by the standard Chernoff argument, for any $\lambda, t > 0$, 
  \begin{align*}
    \P\left[ \sum_{\nu=V/3}^{2V/3} \left( \tilde{\tau}_\nu^2 - \E[ \tilde{\tau}_\nu^2 \mid \cF_{\nu-1} ] \right) \le  -t \right]
    &= \P\left[ 
        e^{- \lambda \sum_{\nu=V/3}^{2V/3} \left( \tilde{\tau}_\nu^2 - \E[ \tilde{\tau}_\nu^2 \mid \cF_{\nu-1} ] \right)} 
        \ge e^{\lambda t} 
      \right] \\
    &\le e^{-\lambda t}
      \E\left[ 
        e^{- \lambda \sum_{\nu=V/3}^{2V/3} \left( \tilde{\tau}_\nu^2 - \E[ \tilde{\tau}_\nu^2 \mid \cF_{\nu-1} ] \right)} 
      \right]. 
  \end{align*}
  For the last expectation, note that $\E[ \tilde{\tau}_{2V/3}^2 \mid \cF_{2V/3-1} ] - \tilde{\tau}_{2V/3}^2$
  is a variable with zero conditional mean and conditional variance bounded by $\sigma_*^2$, and is 
  upper bounded by $\xi_*$ a.s. Hence, by Lemma~\ref{lemma: bennett, intermediate step}, we have, for any 
  $\lambda \le \xi_*\inv$,  
  \begin{align*}
    \E\left[ 
      e^{- \lambda \sum_{\nu=V/3}^{2V/3} \left( \tilde{\tau}_\nu^2 - \E[ \tilde{\tau}_\nu^2 \mid \cF_{\nu-1} ] \right)} 
    \right]
    &= \E\left[ 
        e^{- \lambda \sum_{\nu=V/3}^{2V/3-1} \left( \tilde{\tau}_\nu^2 - \E[ \tilde{\tau}_\nu^2 \mid \cF_{\nu-1} ] \right)} 
        \E\left[
          e^{\lambda \left( \E[ \tilde{\tau}_{2V/3}^2 \mid \cF_{2V/3-1} ] - \tilde{\tau}_{2V/3}^2 \right)} 
          \;\bigg|\; 
          \cF_{2V/3-1}
        \right]
      \right] \\
    &\le 
      \E\left[ 
        e^{- \lambda \sum_{\nu=V/3}^{2V/3-1} \left( \tilde{\tau}_\nu^2 - \E[ \tilde{\tau}_\nu^2 \mid \cF_{\nu-1} ] \right)} 
      \right] 
      e^{\lambda^2 \sigma_*^2}. 
  \end{align*}
  Repeat the same process and we get 
  \[
    \P\left[ \sum_{\nu=V/3}^{2V/3} \left( \tilde{\tau}_\nu^2 - \E[ \tilde{\tau}_\nu^2 \mid \cF_{\nu-1} ] \right) \le  -t \right]
    \le \exp\left( -\lambda t + V \lambda^2 \sigma_*^2 / 3 \right).
  \]
  Choose $t = \beta \norm{\z}^2/ m $ where $\beta > 0$ is a parameter to be determined later. The RHS is minimized 
  at 
  \[
    \lambda_* 
    := \frac{3 t}{2 V \sigma_*^2}
    = \frac{3 t m^2 V}{2 C \log^2( V / \delta_{\P}) \norm{\z}^4 }
    = \frac{3 \beta m V}{2 C \log^2( V / \delta_{\P}) \norm{\z}^2 }.
  \]
  For it to be admissible, it suffices to require 
  \[
    \lambda_*
    \le \frac{1}{\xi_*}
    \quad\Leftarrow\quad 
    \frac{3 \beta m V}{2 C \log^2( V / \delta_{\P}) \norm{\z}^2 } 
    \le \frac{m V}{C \log(V / \delta_{\P}) \norm{\z}_2^2 }
    \quad\Leftarrow\quad 
    \beta 
    \le \frac{\log(V / \delta_{\P})}{3}.
  \]
  In particular, this implies $\lambda_*$ is admissible for all $\beta = O(1)$. As a result, 
  \begin{align*}
    \P\left[ \sum_{\nu=V/3}^{2V/3} \left( \tilde{\tau}_\nu^2 - \E[ \tilde{\tau}_\nu^2 \mid \cF_{\nu-1} ] \right) \le  - \frac{\beta \norm{\z}^2}{m} \right]
    \le \exp\left( 
        - \frac{3 t^2}{4 V \sigma_*^2} 
      \right)\bigg|_{t=\beta\norm{\z}^2/m} 
    = \exp\left( - \frac{3 \beta^2  V}{ 4 C \log^2( V / \delta_{\P}) } \right). 
  \end{align*}
  For the last term to be bounded by $\delta_{\P}$, it suffices to require $\beta^2  V \gg \log^3(V/\delta_{\P})
  \sim (\log\log d + \log(1/\delta_{\P}))^3$.

  Finally, we estimate $\sum_{\nu=V/3}^{2V/3} \E[ \tilde{\tau}_\nu^2 \mid \cF_{\nu-1} ]$. Our goal is to show 
  it is $\Omega( \norm{\z}^2 / m )$ with high probability. 
  By Lemma~\ref{lemma: ynu2 conditional mean and var concentration} and union bound, we know, with probability at least 
  $1 - \delta_{\P}$, that 
  \begin{align*}
    \sum_{\nu=V/3}^{2V/3} \E[ \tilde{\tau}_\nu^2 \mid \cF_{\nu-1} ]
    = \sum_{\nu=V/3}^{2V/3} \E[ \tau_\nu^2 \mid \cF_{\nu-1} ] \iota_\nu 
    &= \sum_{\nu=V/3}^{2V/3} \E[ \tau_\nu^2 \mid \cF_{\nu-1} ] \\
    &\ge \left( 1 \pm \frac{60\log(2 V / \delta_{\P})}{m } \right) 
      \frac{\norm{\z}^2}{5 m V}
    \ge \frac{\norm{\z}^2}{10 m}. 
  \end{align*}
  Pick $\beta = 1/20$. Combining the above two bounds, we conclude that 
  \[
    \norm{\btau}^2 
    \ge \frac{\norm{\z}^2}{20 m}, \quad 
    \text{with probability at least }
    1 - O(\delta_{\P}). 
  \]
\end{proof}

\subsection{Multistep concentration}

Recall that our goal is to lower bound the norm of $\btau\ps{l}_\mu - \btau\ps{l}_{\mu'} = \bP_1 \cdots \bP_l (\e_{\mu} - \e_{\mu'})$ 
for $\mu \ne \mu' \in \cV_l$. In addition, since $(\bP_k)_k$ are column stochastic, $\One\trans\bP_k\x = 0$, as long as 
$\One\trans\x = 0$. In particular, this implies $\One\trans (\bP_k \cdots \bP_l (\e_{\mu} - \e_{\mu'})) = 0$. As a 
result, we can repeatedly apply Lemma~\ref{lemma: rhm stats, one step concentration} to unroll 
$\norm{ \bP_1 \cdots \bP_l (\e_{\mu} - \e_{\mu'}) }$. Formally, we have the following lemma. 

\begin{lemma}
  \label{lemma: rhm, multistep concentration}
  Suppose that $m, V = \Theta(\log d)$ and $L = O(\log d / \log\log d)$.
  Fix $l \in [L]$ and $\mu \ne \mu' \in \cV_l$. As long as $\min\{m, V\} \gg \log^3(1/\delta_{\P})$, then we have 
  $\norm{ \btau\ps{l}_\mu - \btau\ps{l}_{\mu'} } \ge (20 m)^{-l/2}$ with probability at least $1 - O(\delta_{\P})$.
\end{lemma}
\begin{proof}
  First, recall from Lemma~\ref{lemma: P to balls into bins} that $\bP_1, \dots, \bP_L$ are i.i.d.~copies of $\bP$. 
  Hence, by Lemma~\ref{lemma: rhm stats, one step concentration} and union bound, we have with probability at least 
  $1 - O(L \delta_{\P})$ that 
  \[
    \norm{\bP_1 \cdots \bP_l (\e_{\mu} - \e_{\mu'})}^2 
    \ge \frac{1}{20 m} \norm{\bP_2 \cdots \bP_l (\e_{\mu} - \e_{\mu'})}^2 
    \ge \cdots 
    \ge (20 m)^{-l} \norm{\e_\mu - \e_{\mu'}}^2
    \ge (20 m)^{-l},
  \]
  as long as $\min\{m, V\} \gg \log^3(1/\delta_{\P})$. To complete the proof, it suffices to replace $\delta_{\P}$ with 
  $\delta_{\P}/L$ and note that $\log^3 L \le (\log\log d)^3 \lesssim \log d$. 
\end{proof}

We are now ready to prove Proposition~\ref{prop: signal lower bound}, which we restate below. 
\rhoLowerBound*
\begin{proof}
  Let $Q_l: \cV_{l-1} \times \cV_l^s \to \R_{\ge 0}$ denote the transition function at level $l$. That is, 
  for any $\nu \in \cV_{l-1}$ and $\bmu \in \cV_l^s$, $Q_l(\nu, \bmu)$ is the probability of $\nu$ generating $\bmu$. 
  We also abuse notations and write $\Q_l \in \R^{V \times V^s}$ for the corresponding transition matrix. 
  Then, for any label $\zeta \in \cV_0$ and level-$l$ patch $\bmu \in \cV_l^s$, we have 
  \[
    \P\left[ \text{first level-$l$ patch is $\bmu$} \mid \text{label is $\zeta$} \right]
    = \e_\zeta\trans \bP_1 \cdots \bP_{l-1} \Q_l \e_{\bmu}. 
  \]
  Since the marginal distribution of the first token at each level is uniform, 
  $\P\left[ \text{first level-$l$ patch is $\bmu$} \right] = 1/(mV)$. 
  Hence, by the Bayes rule, for any possible level-$l$ patch $\bmu \in \cP_l$, we have 
  \begin{align*}
    \P\left[ \text{label is $\zeta$} \mid \text{first level-$l$ patch is $\bmu$} \right]
    &= \frac{
        \P\left[ \text{first level-$l$ patch is $\bmu$} \mid \text{label is $\zeta$} \right]
        \P\left[ \text{label is $\zeta$} \right]
      }{
        \P\left[ \text{first level-$l$ patch is $\bmu$} \right]
      } \\
    &= \frac{
        \P\left[ \text{first level-$l$ patch is $\bmu$} \mid \text{label is $\zeta$} \right]
        /V
      }{
        1 / (mV)
      } \\
    &= m \e_\zeta\trans \bP_1 \cdots \bP_{l-1} \Q_l \e_{\bmu}.
  \end{align*}
  In other words, for any $l \in [L]$ and $\bmu \in \cP_l$, we have 
  \[
    \q\ps{l}_{\bmu}
    = m \e_\zeta\trans \bP_1 \cdots \bP_{l-1} \Q_l \e_{\bmu}. 
  \]
  Since the RHM is assumed to be $(V, m)$-uniform, we have $Q_l(\nu, \bmu) = \indi\{ \bmu \in \cP_{l, \nu} \} / m$.
  In particular, this implies $\Q_l \e_{\bmu} = \e_{\nu_{\bmu}}$, where $\nu_{\bmu} \in \cV_{l-1}$ is the unique 
  level-$(l-1)$ symbol that can generate $\bmu$. Thus, for any pair of non-synonyms $\bmu, \bmu' \in \cP_l$, we have 
  \[
    \q\ps{l}_{\bmu} - \q\ps{l}_{\bmu'}
    = m \e_\zeta\trans \bP_1 \cdots \bP_{l-1} \left( \e_{\nu_{\bmu}} - \e_{\nu_{\bmu'}}  \right). 
  \]
  Since $\bmu, \bmu'$ are non-synonyms, $\| \e_{\nu_{\bmu}} - \e_{\nu_{\bmu'}} \|^2 = 2$. Thus, by 
  Lemma~\ref{lemma: rhm, multistep concentration}, $\norm{ \q\ps{l}_{\bmu} - \q\ps{l}_{\bmu'} } \ge (20 m)^{-(l-1)/2}$
  with probability at least $1 - O(\delta_{\P})$, as long as $\min\{m, V\} \gg \log^3(1/\delta_{\P})$. Take union bound 
  over $l \in [L]$ and $\bmu, \bmu' \in \cP_l^2$, and we complete the proof.
\end{proof}

\newpage
\section{Application of Informal Principle~\ref{principle: shallow to deep chaining}: Deep quadratic boolean functions}
\label{sec: deep quad}

We include in this section another application of Informal Principle~\ref{principle: shallow to deep chaining}
to show its generality. The argument here will be informal. In particular, the example distribution considered here contains all the structure in the targets while having unstructured inputs. 
This can be viewed as a simplified boolean version of the target class considered in \cite{allen-zhu_backward_2023}.
Hierarchical learning for a related class of tree-structured targets defined on Gaussian inputs was also considered in \cite{dandi_computational_2025}.

Let $\{\pm 1\}^d$ be our input space. For each $S \subset [d]$, define $\chi_S: \{\pm 1\}^n \to \R$ to be the 
multilinear monomial $\chi_S(\x) = \prod_{i \in S} x_i$, with $\chi_{\varnothing} \equiv 1$ by convention. 
Equip the input space with the uniform distribution $\cD := \Unif(\{\pm 1\}^d)$. It is well-known that 
$\{ \chi_S \}_{S \subset [d]}$ is an orthonormal basis of $L^2(\cD)$, which is often called the Fourier basis 
of boolean functions (see, for example, \cite{odonnell_analysis_2014}). 
Therefore, any $f : \{\pm 1\}^d \to \R$ can be uniquely decomposed as $f = \sum_{S \subset [d]} \hat{f}_S \chi_S$, 
where $\hat{f}_S := \inprod{f}{\chi_S}_{L^2(\cD)} = \E_{\x \sim \cD}[ f(\x) \chi_S(\x) ]  \in \R$. 

Provided that all nonzero Fourier coefficients $\hat{f}_S$ are bounded away from $0$, learning $f$ can be reduced to 
identifying the support $\supp\hat{f} := \{ S \subset [d] \,:\, \hat{f}_S \ne 0 \}$, since with $\supp\hat{f}$ known,
the corresponding coefficients $\hat{f}_S = \E_{\x \sim \cD}[ f(\x) \chi_S(\x) ]$ can be efficiently estimated. 

In this section, we consider deep quadratic (boolean) functions, which is a special class of boolean functions whose 
support has a binary tree/forest structure. Roughly speaking, we consider functions of form $f(\x) = x_1 x_2 + x_3 x_4 
+ (x_1 x_2) (x_3 x_4) + x_5 x_6$. That is, the support of the higher-order terms can be obtained by joining two 
lower-order terms, and within each level, the supports of different terms are disjoint. See Figure~\ref{fig: example, 
deep quad} for another example. 

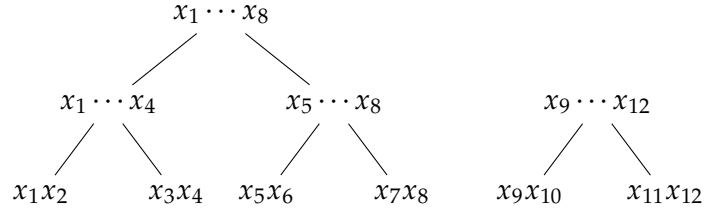
\begin{figure}[htbp]
  \centering
  \begin{tikzpicture}[
    grow=down,
    level distance=12mm,
    every node/.style={},
    edge from parent/.style={draw}
  ]

  \begin{scope}[
    xshift=0cm,
    level 1/.style={sibling distance=30mm},
    level 2/.style={sibling distance=18mm}
  ]
  \node {$x_1 \cdots x_8$}
    child { node {$x_1 \cdots x_4$}
      child { node {$x_1 x_2$} }
      child { node {$x_3 x_4$} }
    }
    child { node {$x_5 \cdots x_8$}
      child { node {$x_5 x_6$} }
      child { node {$x_7 x_8$} }
    };
  \end{scope}

  \begin{scope}[
    xshift=5cm,
    yshift=-12mm,
    level 1/.style={sibling distance=18mm}
  ]
  \node { $x_9 \cdots x_{12}$ }
    child { node {$x_9 x_{10}$}
    }
    child { node {$x_{11} x_{12}$}
    };
  \end{scope}
  \end{tikzpicture}
  \caption{
    Example of a deep quadratic function. The target function is the sum of all terms appearing in the 
    graph. That is, $f(\x) = x_1 \cdots x_8 + x_9 \cdots x_{12} + x_1 \cdots x_4 + x_5 \cdots x_8 + x_1x_2 + \cdots x_7 x_8$.
  }
  \label{fig: example, deep quad}
\end{figure}

Formally, we define our target function class as follows. For ease of presentation, we assume the target function
has no constant and degree-$1$ terms, both of them can be efficiently learned and removed from the target by estimating 
the corresponding $d+1$ Fourier coefficients. 

\begin{definition} 
  \label{def: deep quad}
  Consider $f: \{ \pm 1 \}^d \to \R$ and let $f = \sum_{S \subset [d]} \hat{f}_S \chi_S$ denote its Fourier 
  decomposition. For each $q \in \bbN$, let $\cS_q := \{ S \subset [d] \,:\, |S| = q, \hat{f}_S \ne 0 \}$ denote the 
  supports of the degree-$q$ terms. We say $f$ is a \tnbf{deep quadratic boolean function} if the following hold. 
  \begin{enumerate}[(a)]
    \item $\cS_q \ne \varnothing$ only if $q = 2^k$ for some $k \in \bbN_{>0}$. 
    \item \label{itm: deep quad: disjoint supports} 
      For any $q \in \bbN$, the elements of $\cS_q$ are disjoint. 
    \item \label{itm: deep quad: union}
      For any $S \in \cS_q$ with $q > 2$, there exists $S_1, S_2 \in \cS_{f, q/2}$ such that $S = S_1 \cup S_2$.
  \end{enumerate}
\end{definition}
\begin{remark}
  By condition~\ref{itm: deep quad: disjoint supports}, the decomposition $S = S_1 \cup S_2$ in condition~\ref{itm: deep 
  quad: union} is unique. Moreover, while the support of a general boolean function can have exponentially many terms, 
  the support size of a deep quadratic boolean function is bounded by $O(d)$, as going up one level halves the 
  number of possible terms. 
\end{remark}

Let $f_*: \{\pm 1\}^d \to \R$ denote the target deep quadratic function.
First, we claim $f_*$ can be represented using a deep network. 
For each $D \in \bbN_{>0}$, let $\bPhi: \R^D \to \R^{D(D-1)/2}$ denote the quadratic feature map: 
\[
  \bPhi_D(\x) = ( x_i x_j )_{1 \le i < j \le D} = \{ \chi_S(\x) \}_{S \subset [D], |S|=2}.
\]
For the first layer of our learner network, it suffices to set 
\[
  \y\ps{1} :=  \left( \inprod{\e_S}{ \Phi_d(\x) } \right)_{S \in \cS_1}, \quad 
  f\ps{1} := \sum_{S \in \cS_1} \hat{f}_S y\ps{1}_S.
\]
Note that this is a simple two-layer (linear) network with a nonlinear feature map $\bPhi$. Then, for any $l \ge 2$, 
we inductively construct the $l$-th layer as 
\begin{equation}
  \label{eq: deep quad: learner repr}
  \x\ps{l} := \y\ps{l-1}, \quad 
  \y\ps{l} := \left( \inprod{\e_{S_1S_2}}{\bPhi_{|\cS_{2^{l-1}}|}(\x\ps{l}) } \right)_{S = S_1 \cup S_2 \in |\cS_{2^l}|}, \quad
  f\ps{l} := \sum_{S \in \cS_{2^l}} \hat{f}_S y\ps{l}_S, 
\end{equation}
where $S = S_1 \cup S_2$ is the decomposition guaranteed by condition~\ref{itm: deep quad: union} and $\e_{S_1S_2}$
is the one-hot vector corresponding to that entry in the output of $\bPhi_{|\cS_{l-1}|}$. By construction, 
$f\ps{l}$ is equal to degree-$2^l$ terms in the Fourier decomposition of $f$, and $f := \sum_{l=1}^{\log d} f\ps{l}$
is equal to the target function $f_*$. 

Now, we consider the optimization task and how Informal Principle~\ref{principle: shallow to deep chaining} applies 
to this function class. 

Clear that condition~(i) of Informal Principle~\ref{principle: shallow to deep chaining} holds, as $f_*$ is the summation 
of $\sum_{S \in \cS_{2^l}} \hat{f}_S \chi_S$. 

Consider condition~(ii), which requires the signal received by the lower levels to be clean, so that the lower part of 
the learner will not overfit the higher part of the target. To see how this condition is satisfied, first recall that 
$\{ \chi_S \}_S$ form an orthonormal basis. Then, note that the disjoint supports assumption 
(condition~\ref{itm: deep quad: disjoint supports} of Definition~\ref{def: deep quad}) implies that, if the earlier
layers have identified the correct supports (cf.~\eqref{eq: deep quad: learner repr}), then the inputs of the current 
layer will also be independent Rademacher variables. As a result, once the bottom level is learned, we can essentially 
``remove'' it and reduce the level number of the problem by one. 

Finally, consider condition~(iii) of Informal Principle~\ref{principle: shallow to deep chaining}. Since the features 
$\{ \chi_S \}_S$ are orthonormal, identifiability is no issue in this setting, and there are various ways to enforce it in the learner
model. One such strategy is to choose the $l$-th layer of the learner model to be 
$\x\ps{l} \mapsto \inprod{ \w\ps{l} }{ \bPhi(\x\ps{l}) }$ and train using the MSE loss 
\[
  \Loss\ps{l}(\w\ps{l}) 
  := \hat{\E} \left( f_*(\x) - \sum_{l=1}^{L-1} f\ps{l}(\x) - \inprod{ \w\ps{l} }{ \bPhi(\x\ps{l}) } \right)^2.
\]
Assume (as an induction hypothesis) that $f\ps{l}(\x)$ approximates the degree-$2^l$ terms of $f$ well. Then, 
by our previous discussion on the orthogonality and distribution of $\x\ps{l}$, we know the unique population minimizer 
is $\w\ps{l} = \sum_{S = S_1 \cup S_2 \in \cS_{2^l}} \hat{f}_S \e_{S_1 S_2}$. As long as the nonzero coefficients 
$\hat{f}_S$ are bounded away from $0$, we can approximately recover this population solution using polynomially many 
samples, and then form a feature mapping similar to \eqref{eq: deep quad: learner repr}.

\end{document}